\documentclass[11pt]{article}

\usepackage{amsmath}
\usepackage{amssymb} 
\usepackage{amsthm} 
\usepackage{booktabs}
\usepackage{array}
\usepackage{tabularx}
\usepackage{graphicx}
\usepackage{epstopdf}
\usepackage{float}
\usepackage{ragged2e}
\usepackage[authoryear,round]{natbib} 
\usepackage{tipa}
\usepackage{arabtex}
\usepackage{newunicodechar}
\newunicodechar{ʿ}{\textrevglotstop}    
\newunicodechar{ʊ}{\textupsilon}        
\newunicodechar{Γ}{\ensuremath{\Gamma}} 
\usepackage[colorlinks=true,linkcolor=blue,citecolor=blue,urlcolor=blue]{hyperref}

\newcolumntype{Y}{>{\raggedright\arraybackslash}X}
\title{The No-Meaning Falsity: The Structural Impossibility of the Arbitrary Sign in Classical Arabic}

\author{Elnaserledinellah Mahmoud Abdelwahab\thanks{
    Email: elnaser@hotmail.com.\\
    ORCID iD: \url{https://orcid.org}}\\[8pt]
  \parbox{0.85\linewidth}{\centering\small
    Steering Committee Member, Principal Researcher\\
    GridSAT Stiftung, Germany\\[2pt]
    \url{https://eth.limo}\\[8pt]
    Corresponding Member\\
    Arabic Language Academy, Damascus, Syria}}
\date{}

\makeatletter
\let\BSLArabTeXBegin\begin
\def\begin#1{%
  \def\BSLenv{#1}%
  \def\BSLdocument{document}%
  \ifx\BSLenv\BSLdocument
    \document
    \let\begin\BSLArabTeXBegin
  \else
    \BSLArabTeXBegin{#1}%
  \fi}
\makeatother

\makeatletter
\theoremstyle{plain}
\def\BSLplainthm#1#2{\@ifundefined{#1}{\newtheorem{#1}{#2}}{}}
\BSLplainthm{theorem}{Theorem}
\BSLplainthm{lemma}{Lemma}
\BSLplainthm{corollary}{Corollary}
\theoremstyle{definition}
\BSLplainthm{definition}{Definition}
\BSLplainthm{axiom}{Axiom}
\BSLplainthm{principle}{Principle}
\BSLplainthm{example}{Example}
\BSLplainthm{justification}{Justification}
\theoremstyle{remark}
\BSLplainthm{remark}{Remark}

\makeatother

\newcommand{\keywords}[1]{\par\medskip\noindent\textbf{Keywords:}\ #1\par\bigskip}

\begin{document}

\maketitle


\begin{abstract}
This paper investigates whether the postmodern claim of unrestricted semantic indeterminacy, and its foundational Saussurean axiom of the ``arbitrary sign,'' are compatible with the structural architecture of Classical Arabic. We develop a formal mathematical model of Arabic non-concatenative morphology in which lexical meaning is determined by the interaction between an invariant root and a morphosyntactic pattern. Within this framework, we establish a Morphological Correspondence Theorem, demonstrating that every lexical item is uniquely generated by a root-pattern pair, and a Semantic Localization Theorem, proving that lexical meaning is determined at the derivational level prior to surface realization. Central to our analysis is a purely morphological definition of the ``primitive signifier'' as a single moneme that is syntagmatically irreducible at the first articulation---a classification internal to the structuralist tradition (Saussure, Martinet) that depends solely on form, not on semantics. We then provide the precise mathematical formalization missing from structuralist theory by defining absolute arbitrariness as the maximally strict absence of any systematic form-meaning mapping. Using this definition, we prove a Structural Impossibility Theorem: given the root-pattern architecture and the phonosemantic motivation documented by Ibn Jinni, the universal structuralist claim cannot be sustained---the existence of even a single phonosemantically constrained root logically falsifies it, and the primitive signifier set cannot be coherently defined without contradiction. To address Saussure's weaker notion of relative arbitrariness, we formalize it via conditional Kolmogorov complexity, defining it algorithmically as the ``no-rule'' property. We prove that general relative arbitrariness is formally undecidable, while Arabic relative arbitrariness is decidable and provably less than 1 for its motivated signifiers (Levels W and M), establishing a strict system-complexity asymmetry over Indo-European languages. We further introduce a three-level framework for classifying relative arbitrariness: Level W (word-level), Level M (sub-word/morpheme-level), and Level P (sub-morphemic/phoneme-level). We show that while Indo-European languages are relatively arbitrary at Level W, the structural core of Arabic is relatively motivated at Levels W and M, confining conventionality to Level P---a sub-morphemic level beneath the structuralist definition of a signifier. These synchronic results are complemented by a diachronic analysis of the emergence domain \(D_{\mathrm{em}}\). We show that the Saussurean framework generates a logical contradiction at emergence: roots are abstractions from composite words, so the first signifiers cannot have been roots; and full words are composite, so they cannot be primitive signifiers. The framework simultaneously requires and prohibits primitive signifiers. We then propose and formally verify a constructive alternative \(T\), in which the compositional program is prior to its words. \(T\) is shown to be a consistent explanation of \(D_{\mathrm{em}}\)---the morphological analogue of Chomsky's discrete infinity argument, but weaker in its premises. This establishes a significant epistemological asymmetry: the universal structuralist claim of absolute arbitrariness is false for the structural core of Arabic, while the universal claim of relative arbitrariness is formally undecidable in the general case---rendering relative arbitrariness empirically untestable without an explicit formal model of a given language's morphological system. In contrast, this paper's claims about Arabic are decidable, verifiable, and falsifiable. Consequently, Arabic operates as a motivated, structurally locked system. This imposes mathematical ceilings on Derridean diff\'erance and demonstrates that theories of meaning formulated from concatenative Indo-European languages cannot be applied unchanged to Arabic morphology.
\end{abstract}


\section{Introduction: An Epistemological Problem}
\label{sec:introduction}

The postmodern assertion that texts possess no intrinsic meaning and can be parsed in infinitely arbitrary ways has been treated as a universal law of language. Derived largely from French deconstruction and structuralist linguistics---two distinct but historically intertwined traditions---this framework treats words as weightless, fluid tokens drifting endlessly on the whims of cultural context. This theoretical model is built upon a fundamental assumption: that the relationship between the signifier (the word) and the signified (the concept) is essentially arbitrary \citep{Saussure}. It is crucial to distinguish, however, between Saussure's synchronic structural principle of arbitrariness and Derrida's diachronic concept of \emph{diff\'erance}---the infinite deferral of meaning along an endless chain of signifiers \citep{Derrida}. While distinct, both frameworks share an unformalized assumption about the non-motivation of the sign, and it is this shared assumption that the present work subjects to formal scrutiny. Western literary theory, emerging from Indo-European languages that rely on concatenative morphology, treats, by and large, the word as an atomic unit. Over time, the relationship between this atomic signifier and its signified becomes historically contingent, vulnerable to semantic bleaching, and subject to the perpetual deferral of meaning. 

A note on terminology is required at this early point. In Saussure, arbitrariness is
strictly a property of the \emph{sign}---the union of signifier and signified.
It is the \emph{relation} between the two sides that is arbitrary, not either
side taken alone. Throughout this paper, however, arbitrariness is frequently
predicated of the \emph{signifier relative to its signified}, since the formal
apparatus developed below operates on the signifier and its assigned signified.
This is a notational convenience, not a departure from Saussure: because every
signifier in a sign has an assigned signified, the phrase ``absolutely
arbitrary signifier'' is shorthand for ``a signifier whose relation to its
signified is absolutely arbitrary.'' Similarly for the phrase ``relatively
arbitrary signifier.'' When the paper speaks of ``the arbitrary sign'' without
qualification, it refers to Saussure's doctrine; when it speaks of an
``arbitrary signifier,'' it applies the formal predicates that will be defined
in Sections~\ref{sec:impossibility} and~\ref{sec:kolmogorov}.

However, before any formal analysis can proceed, a more fundamental epistemological problem must be confronted. Ferdinand de Saussure elevated the arbitrariness of the sign to the status of a ``first principle'' of linguistics, declaring it the foundation upon which the entire edifice of linguistic science rests \citep{Saussure}. Yet, crucially, Saussure never provided a formal, operational definition of what constitutes this arbitrariness. The principle is stated as a philosophical intuition---a negative characterization (``there is no natural link'')---rather than a precise, testable criterion.

This omission is not a minor oversight; it renders the principle empirically unassailable. A universal claim that lacks formal criteria for verification or falsification stands refuted epistemologically before any formal or empirical counterexample can even be adduced. One cannot test a hypothesis whose core predicate is undefined. This formal vacuum is all the more striking when situated historically. Saussure's theorizing predates the formalization revolution in linguistics inaugurated by Noam Chomsky in the mid-twentieth century \citep{Chomsky1957}, during which structural and generative properties of language were subjected to rigorous mathematical scrutiny. Yet, despite the proliferation of formal models in generative grammar, computational linguistics, and information theory, Western philosophy of language failed to formalize this foundational axiom.

This failure stands in stark contrast to the mathematical formalization of analogous concepts in other sciences. Unpredictability, disorder, and arbitrariness have been formally captured in physics (thermodynamic entropy), mathematics \citep{Kolmogorov}, and computer science \citep{Shannon}, allowing precise measurement, empirical testing, and rigorous theoretical development. The ``arbitrariness of the sign,'' by contrast, remains a metaphysical assertion immune to empirical challenge---a dogma rather than a hypothesis.

While the absence of a formal criterion for signifier arbitrariness has been repeatedly noted, several research traditions have addressed related questions without supplying the missing formalization. First, Lewis's \emph{Convention} \citep{Lewis1969} provides a prominent formal treatment of linguistic convention, demonstrating via signaling games how arbitrary conventions emerge and stabilize as rational equilibria in populations. However, Lewis formalizes the \emph{dynamics} of conventional behavior---the reasons why arbitrary sign relations persist---rather than the \emph{criterion} for determining whether a specific signifier's form is motivated relative to its meaning. His framework takes arbitrariness as a given and explains its social consequences; it does not provide the testable definition of arbitrariness as a property of individual signs that structuralist theory requires. Second, empirical work on sound symbolism \citep{Blasi2016} has tested cross-linguistic phoneme--meaning associations at the phonetic level, revealing systematic but weak biases. This research addresses a different question---whether certain phonetic features tend to correlate with certain meanings across languages---rather than whether the morphological structure of a specific language renders its signifiers compressible relative to its meaning system. Third, the Minimum Description Length (MDL) tradition in computational morphology \citep{Rissanen1978,Goldsmith2001} applies compressibility principles to the problem of \emph{learning} morphological systems from raw data. While this approach shares our commitment to compression as a measure of linguistic structure, it treats the morphological system as the output of an inference procedure rather than as established background knowledge relative to which individual signifiers are analyzed. 

What none of the above-mentioned Western traditions supplies, however, is precisely what structuralist theory lacks: a formal, operational definition of arbitrariness as a property of individual signifiers that can be tested against the morphological architecture of a specific language. This is the gap the present paper addresses. 

This paper argues that the application of this formally under-defined structuralist framework to Classical Arabic constitutes an epistemological overreach of the highest order. Classical Arabic is built upon a highly structured, non-concatenative matrix where meaning is mathematically calculated at the structural level. Language in this framework is governed by the absolute interplay between an immutable root (\textit{jazr}) which carries the permanent semantic DNA of a word---and a precise functional pattern (\textit{wazn}) that dictates its grammatical geometry. To apply flat-surface, arbitrary-sign theories to such a mathematically consistent language is to ignore the physical rules of Arabic morphology. It substitutes a localized European intellectual paradigm for a universal law without empirically testing it against structurally distinct data.

Separately from these traditions, Western Semitic linguistics has also debated the ontological status of the Arabic root, a central constituent of Arabic morphology. Some approaches treat it as a semantic entity carrying inherent meaning, either at the triliteral level or at a deeper biconsonantal level \citep{Bohas2000}. Others treat it as a morphological abstraction emerging from word-based derivational processes \citep{Ratcliffe1998}, or as an epiphenomenon of prosodic constraints \citep{Ussishkin2005}. 

This debate must be distinguished from a prior and more fundamental question on which the classical Arabic grammatical tradition, adopted by the current work, is not divided. The tradition uniformly holds that a given root is linked to a semantic meaning: the root \RL{k-t-b}, for instance, is definitely associated with writing. What is contested, within the tradition itself, is Ibn Jinni's stronger conjecture of \emph{al-ishtiqāq al-kabīr} \citep{IbnJinni}: that different consonantal permutations of the same root share an abstract semantic field. The present work formalizes this stronger conjecture in Definition~\ref{def:phonosemantic}. 

It is important to note how this formalization is used. The phonosemantic function \(h\) serves as a formal compression algorithm, not as an empirical claim about permutation semantics. The relative arbitrariness results of Section~\ref{sec:kolmogorov} are therefore robust to the empirical validity of Ibn Jinni's conjecture: the compression partition is used as a formal device, and the results hold regardless of whether the partition reflects a genuine semantic organization applicable to all roots. The Structural Impossibility Theorem, by contrast, does rely on the existence of at least one phonosemantically constrained root as a counterexample to the universal arbitrariness claim. This distinction is made explicit throughout the paper.

To formally prove our theorems, we must first provide the very formalization that is missing. We model Arabic morphology mathematically, then define the concept of ``absolute arbitrariness'' rigorously, and finally demonstrate the structural impossibility of the arbitrary sign in Arabic. We then introduce Kolmogorov complexity to formalize Saussure's weaker notion of relative arbitrariness, showing that the degree of relative motivation of a signifier is inversely proportional to its Kolmogorov complexity under the generative grammar of the language: a signifier is relatively motivated to the extent that its structure can be compressed by the language's morphological rules. We prove that the general case is undecidable while Arabic provides a decidable counterexample.

We also introduce a three-level framework for classifying relative arbitrariness in natural languages---Level W (word-level), Level M (sub-word/morpheme-level), and Level P (sub-morpheme/phoneme-level)---and show that Arabic is relatively motivated at Levels W and M, with relative arbitrariness confined to Level P, two levels deeper than Indo-European languages.

To motivate this discussion further, it suffices to notice that in English (as well as related Indo-European languages) the surface structure of the majority of words, like ``dog'', ``cat'', ``bird'', etc., is unmotivated, i.e. there are no morphemes holding meaning particles which can explain why such specific structure is linked to that specific meaning. It is true that some words, like ``reader'', ``teacher'', or ``driver'' for example, possess morphemes (here: \textit{-er}) which are linked to specific meaning-roles (here: agent) but such words are the exception, not the rule. Such words were recognized by Saussure himself to possess a relative arbitrary relation between signifier and signified. On the other hand: the majority of Arabic words are formed from concrete and well-known root- and pattern-morphemes, which contribute both to forming the exact meaning of a word. Take the Arabic word for ``dog'' for example, which is ``kalb''. Its root is ``k-l-b'' which can also be used in many other patterns to form related meanings: ``takalub'' means ``to crave after'' (like a dog usually craves after a prey or a toy), ``kallabat'' means ``clutches'' (like a dog usually clings to things it takes in its mouth), etc. One does not expect English to use a word like ``dogg-es'' instead of ``clutch-es'', but Arabic allows exactly that, because there is the root ``k-l-b''. It is thus fair to say that most words in English possess W-level relative arbitrariness, while most words of Arabic do not.

As with any formalism, which is always only an approximation of a studied phenomenon, our own idealizes this above general rule and ignores existing exceptions in both languages. When we refer below to the formalized language \(\mathcal{L}_{\mathcal{M}}\), we shall thus mean: the subset of Arabic vocabulary whose morphological structure is fully analyzable into a tri/quadri-consonantal root and a vocalic pattern, excluding loanwords, fossilized forms, and irregular lexical items---i.e., the structural core of the language.

A further crucial distinction must be underlined at the outset: there are at least two distinct modes of defining ``arbitrariness'' of the relation between signifier and signified---an abstract, absolute mode and an algorithmic, complexity-theoretic mode. Saussure's foundational claim that the sign is ``absolutely arbitrary'' posits an ontological universal negative, asserting that no inherent natural link exists between a physical signifier and its concept. Conversely, his secondary concept of ``relative arbitrariness'' implies a spectrum of motivation that aligns naturally with algorithmic information theory, where the structure of a signifier can be compressed by the rules of the linguistic system.

\noindent The formalization undertaken in the present work is offered as an \emph{explicative rational reconstruction} of certain foundational concepts implicit in Saussurean structuralism---specifically, the distinctions between absolute and relative arbitrariness, the morphological categorization of simple versus complex syntagms, and the structural status of lexemic monemes. It must be stated preemptively, however, that any objection which dismisses this formalization on the grounds that ``Saussure himself never formalized these concepts'' commits a category error, for such an objection is both logically vacuous and epistemologically unsound. If the mere absence of formalization in the original text renders all subsequent formal attempts illegitimate, then the theory is effectively sealed against rational critique or clarification, rendering it unfalsifiable and thus unscientific---a position directly contradicted by Saussure's own ambition to establish linguistics as a genuine \emph{s\'emiologie} or science of signs \citep[p.~81]{Saussure}. Furthermore, a legitimate accusation of ``straw-man'' misrepresentation requires the existence of a fixed, formal original against which the formalization can be measured; since Saussure never produced such a formal system, and since critics cannot supply a canonical counter-formalization, the charge of ``straw-man'' is not merely unproven but strictly inapplicable. In the philosophy of science, our undertaking falls under the category of rational reconstruction \citep[p.~3]{Carnap1950}, whereby implicit, informal doctrines are rendered explicit and testable for the purpose of internal coherence and logical consistency. Consequently, the formalization developed below does not claim to represent Saussure's own private mental constructs; rather, it claims to be a rigorous, charitable, and intersubjectively testable rendering of what his written distinctions logically entail when interpreted through the lens of set-theoretic and predicate-logical apparatus. To reject formalization on the basis of authorial silence is not a refutation but an evasion---and it is precisely this evasion that the present introduction seeks to disarm by asserting that formal explication is not an act of misrepresentation, but an act of epistemological hygiene.

\noindent The remainder of this paper is organized as follows. Section~2 provides a formal explication of the morphological system of Arabic and specifies what is meant by a ``primitive signifier'' in the context of this paper. Section~3 defines Saussure's notion of absolute arbitrariness using set-theoretic and predicate-logical apparatus and proves the Structural Impossibility Theorem. Section~4 formalizes relative arbitrariness via Kolmogorov complexity, introduces the three-level framework (W, M, P) for classifying relative arbitrariness across languages, and proves that Arabic provides a decidable counterexample to the universal relative arbitrariness thesis. Section~5 examines the Chomskyan amendment to Saussure and the diachronic frameworks that inherit it, assessing whether they succeed in preserving the Saussurean assumption. Section~6 briefly outlines how the classical Arabic linguistic tradition approached the question of arbitrariness. Section~7 then synthesizes the formal arguments of the preceding sections in a systematic, symbolic form. It distinguishes two logically distinct domains: the synchronic domain \(D_{\mathrm{syn}}\), in which the Saussurean assumptions are shown to generate contradiction or undecidability, and the emergence domain \(D_{\mathrm{em}}\), in which no candidate satisfies the  primitive-signifier conditions. It then proposes and formally verifies a constructive alternative emergence account \(T\), in which the compositional program is prior to the words it generates. Section~8 concludes with the broader epistemological and philosophical implications of the results.

\section{Mathematical Formalization of Arabic Morphology}
\label{sec:formalization}

\subsection{Formal Definition: The ``Primitive Signifier'' (Structural-Morphological Sense)}

The concept of a ``primitive signifier'' is not a standard term within structuralism itself; rather, it is a provocative idea that structuralist and post-structuralist thinkers systematically dismantle from various angles. From the foundational linguistic principles of Ferdinand de Saussure to the anthropological critiques of Claude L\'evi-Strauss, the structuralist tradition argues that the very notion is theoretically impossible. Saussure established that the relationship between a signifier (the sound-image or written form) and its signified (the concept) is arbitrary and unmotivated. Crucially, meaning does not inhere in any single signifier but is generated through its \emph{difference} from all other signifiers within the language system; as Saussure famously concluded: In the linguistic system there are only differences, without positive terms \citep[p.~120]{Saussure}. L\'evi-Strauss extended this logic, concluding that there is nothing truly ``primitive'' about so-called primitive languages or the people who speak them, as any culture capable of evolving a complex system of signs possesses a fully developed and equally sophisticated structure \citep{LeviStrauss1958,LeviStrauss1962}. Linguist \'Emile Benveniste further reinforced this structuralist position by refuting Freud's thesis on ``primitive words,'' arguing that to envision a primitive stage of language, no matter how rudimentary, is ``a pure chimera'' because language, as a system and a structure, possesses a universal and ahistorical quality \citep{Benveniste1966}. The closest functional analogue to a ``primitive'' signifier emerges in the psychoanalytic theory of Jacques Lacan, where the phallus operates as a ``transcendental signifier'' and the fundamental mediator of the subject's primordial identification with the Other \citep{Lacan1966}. This is not a biological organ but a signifier of absence, desire, and the law, which is said to have priority in relation to the signified \citep{Lacan1966}. However, this Lacanian move is precisely what post-structuralist Jacques Derrida critiques. Derrida argues that by positing the phallus as a master signifier, Lacan merely substitutes one form of metaphysical ``origin'' for another, effectively creating a ``transcendental signified.'' For Derrida, meaning is never anchored by a final or primitive signifier; it is generated through the endless play of \emph{diff\'erance}, where every signifier points to another in an infinite chain. In place of an origin, Derrida proposes the ``trace'' as a non-originary origin, describing \emph{diff\'erance} as ``the non-full, non-simple, structured and differentiating origin of differences. Thus the name `origin' no longer suits it'' \citep{Derrida}, deconstructing any notion of the primitive as a constructed, non-primitive idea. 

\noindent Thus, the term ``primitive signifier'' serves not as a technical designation but as a critical lens through which structuralism and its successors reveal the impossibility of a foundational, self-sufficient meaning at the heart of any language system.

\noindent However, if we formally define a ``primitive signifier'' as a word composed of one lexeme only---a single indivisible lexical moneme with no internal concatenative structure---then structuralist morphology unequivocally recognizes such units as the terminal points of segmentation at the first articulation; ``cat'' (or ch\^ene) qualifies as structurally primitive, while ``teacher'' (or poirier) does not because it is divisible into a lexical base and a grammatical suffix \citep{Martinet1960}; \citep[p.~132]{Saussure}. However, this purely formal ``primitivity'' must never be confused with semantic foundationalism; Saussure explicitly insists that even these irreducible radicals possess no positive or intrinsic meaning, deriving their value solely through differential opposition to all other radicals within the closed linguistic system \citep[pp.~120,~133]{Saussure}. Therefore, within the structuralist framework, such definition of ``primitive signifier'' is a valid and precise formal-structural designation---but it is simultaneously a semantic non-entity, as meaning never originates in any single signifier, no matter how morphologically simple it may be. The primitive is structural; the arbitrary is absolute; and the differential is universal \citep[p.~120]{Saussure}; \citep{Martinet1960}.

Let \( S \) denote a linguistic sign (a word). Let \( M \) be the set of all monemes (minimal meaningful units of the first articulation), partitioned into lexical monemes \( \textit{Lex} \subset M \) and grammatical monemes \( \textit{Gram} \subset M \). Let \( \Sigma \) denote the syntagmatic concatenation operator, and let \( \Phi \) be the set of phonemes (units of the second articulation, devoid of meaning).

A sign \( S \) is defined as a \textbf{``primitive signifier''} if and only if it satisfies the following four formal criteria:

\begin{enumerate}
    \item \textbf{Monemic Singularity:} \( S \) consists of exactly one moneme. 
    \[
    \exists m \in M \text{ such that } S = m .
    \]

    \item \textbf{Lexemic Type:} The moneme \( m \) belongs to the lexical subset.
    \[
    m \in \textit{Lex}.
    \]

    \item \textbf{Syntagmatic Irreducibility:} \( S \) cannot be segmented into a sequence of two or more consecutive monemes.
    \[
    \nexists \{m_1, m_2\} \subset M \text{ such that } m_1 \Sigma m_2 = S \text{ with } m_1 \neq S \text{ and } m_2 \neq S.
    \]

    \item \textbf{Articulatory Terminality:} Any further segmentation of \( S \) yields only phonemes, which carry no meaning.
    \[
    \text{If } S \text{ is segmented into } s_1, s_2, \dots, s_n, \text{ then } \forall s_i \in \Phi \text{ and } s_i \notin M.
    \]
\end{enumerate}

\textbf{Semantic Corollary (Structuralist Axiom):}
Despite satisfying (1)--(4), the signified of \( S \) possesses no intrinsic or positive meaning. Its value is purely differential and defined by opposition to all other signs in the closed linguistic system:
\[
\text{Value}(S) = \text{Diff}(S, M \setminus \{S\}), \quad \text{hence } \text{Value}(S) \neq \text{PositiveContent}(S).
\]
\citep[pp.~120,~133]{Saussure}; \citep{Martinet1960}.

\textbf{Examples:}
\begin{table}[H]
\centering
\caption{Examples of primitive vs. non-primitive signifiers in concatenative morphology.}
\begin{tabularx}{\linewidth}{@{}Y Y Y Y Y@{}}
\hline
Sign \( S \) & Monemic Count & \( m \in \textit{Lex} \)? & Internal Concatenation? & Primitive Signifier? \\
\hline
``cat'' (\textit{ch\^ene}) & 1 & Yes & No (\( \nexists\, m_1 \Sigma m_2 \)) & Yes \\
\hline
``teacher'' (\textit{poirier}) & 2: \textit{teach} + \textit{-er} & Yes + No (\textit{-er} \(\in\) \textit{Gram}) & Yes (\( \textit{teach} \Sigma \textit{-er} \)) & No \\
\hline
``cats'' & 2: \textit{cat} + \textit{-s} & Yes + No (\textit{-s} \(\in\) \textit{Gram}) & Yes (\( \textit{cat} \Sigma \textit{-s} \)) & No \\
\hline
``re-read'' (\textit{relire}) & 2: \textit{re-} + \textit{read} & No (\textit{re-} \(\in\) \textit{Gram}) + Yes & Yes (\( \textit{re-} \Sigma \textit{read} \)) & No \\
\hline
\end{tabularx}
\end{table}

\paragraph{Formal Statement of the Consequence}
Let \( \mathcal{P} \) be the set of all primitive signifiers as formally defined. Then:
\[
\forall m \in \mathcal{P}, \quad m \text{ is absolutely arbitrary.}
\]
\[
\nexists m \in \mathcal{P} \text{ such that } m \text{ is relatively arbitrary.}
\]
\textbf{Proof:} Relative arbitrariness requires \( m \) to have an internal syntagmatic structure \citep[p.~132]{Saussure}. However, by Criterion 3 of our definition, \( m \) has no internal syntagmatic structure. Therefore, \( m \) cannot be relatively arbitrary. By the structuralist dichotomy, it must be absolutely arbitrary \citep[pp.~131--133]{Saussure}.

\textbf{Summary:}
Within the structuralist apparatus of the double articulation, a \textit{primitive signifier} is formally defined as a lexemic moneme that constitutes a complete word, is syntagmatically irreducible, and terminates morphological analysis at the boundary of the first articulation. This formal primitivity, however, entails no semantic privilege; its signified remains purely differential, negative, and systemically constituted through opposition to all other monemes in the language \citep[pp.~120,~132--133]{Saussure}; \citep[p.~22]{Martinet1960}.
\subsection{Formal Definition: Arabic Morphological System}

The formalization of the primitive signifier in §2.1 is predicated upon concatenative morphology---the linear segmentation of surface strings into consecutive monemes. This framework is structurally adequate for Indo-European languages but is \emph{inapplicable} to non-concatenative systems like Arabic. In Arabic, the minimal meaningful units (roots and patterns) are not linearly segmentable monemes; they are interdigitating, discontinuous morphemes. Consequently, the core of Arabic does not possess surface primitive signifiers (as defined in §2.1) at the word level. The following formalization captures this structural distinction and provides the necessary apparatus to analyze the locus of meaning in Arabic.

To rigorously analyze the locus of meaning in Arabic, we must first formalize its morphological components.

\begin{axiom}[The Root Space]
\label{ax:root}
Let \(\mathcal{R}\) be the set of all valid triliteral/quadrilateral roots. A root \(R \in \mathcal{R}\) maps to a fixed semantic field (domain of core meaning) denoted by \(\sigma_R(R)\).
\begin{equation}
\sigma_R : \mathcal{R} \to S_R
\label{eq:root-mapping}
\end{equation}
\end{axiom}

\begin{axiom}[The Pattern Space]
\label{ax:pattern}
Let \(\mathcal{P}\) be the set of all valid morphosyntactic patterns (\textit{Awzan} \RL{'awzAn}). A pattern \(P \in \mathcal{P}\) maps to a structural, aspectual, or functional modification denoted by \(\sigma_P(P)\).
\begin{equation}
\sigma_P : \mathcal{P} \to S_P
\label{eq:pattern-mapping}
\end{equation}
\end{axiom}

\begin{definition}[Arabic Morphological System]
\label{def:morph-system}
An Arabic morphological system is a quadruple \(\mathcal{M} = \langle \mathcal{R}, \mathcal{P}, \iota, \rho \rangle\) where:
\begin{itemize}
    \item \(\mathcal{R}\) is the set of roots;
    \item \(\mathcal{P}\) is the set of morphological patterns;
    \item \(\iota : \mathcal{R} \times \mathcal{P} \to \Sigma^{*}\) is the interdigitation operator;
    \item \(\rho : \Sigma^{*} \to \Sigma^{*}\) is the phonological rewriting operator.
\end{itemize}
The morphology operator is defined by \(\mathcal{M} = \rho \circ \iota\).
\end{definition}

The formalization of non-concatenative morphology has been previously attempted in Western linguistics, notably through autosegmental frameworks \citep{McCarthy}, though these models stop short of algorithmic information theory.

\begin{definition}[Image of the Morphology Operator]
\label{def:image}
The image of the morphology operator, denoted \(\operatorname{Im}(\mathcal{M})\), is the set of all valid surface strings generated by the system:
\begin{equation}
\operatorname{Im}(\mathcal{M}) = \{w \in \Sigma^{*} \mid \exists (R,P) \in \mathcal{R} \times \mathcal{P}, \mathcal{M}(R,P) = w\}
\label{eq:image}
\end{equation}
\end{definition}

\begin{definition}[Morphological Language]
\label{def:morph-language}
The language generated by \(\mathcal{M}\) is defined as the image of the morphology operator, such that \(\mathcal{L}_{\mathcal{M}} = \operatorname{Im}(\mathcal{M})\).
\end{definition}

Let \( S \) be the set of all composite lexical meanings generated by \( F \), i.e., \( S = \{ F(s_R, s_P) \mid s_R \in S_R, s_P \in S_P \} \).

\begin{definition}[Semantic Assignment]
\label{def:semantic-assignment}
The semantic interpretation of every generated word is defined by:
\begin{equation}
\mu(\mathcal{M}(R,P)) = F(\sigma_R(R), \sigma_P(P))
\label{eq:semantic}
\end{equation}
where \( F : S_R \times S_P \to S \) is the semantic composition operator.
\end{definition}

\begin{definition}[Strong Surface Semantic Assignment]
\label{def:sssa}
A language satisfies the Strong Surface Semantic Assignment hypothesis (SSSA) iff lexical meaning is assigned directly to surface words. Equivalently, every newly generated lexical item receives its lexical meaning only after the corresponding surface word has been formed.
\end{definition}

\subsection{Formal Semantics and Structural Correspondence}
\label{sec:formal-semantics}

Using the formal definitions, we establish two critical theorems regarding the architecture of the Arabic language.

\begin{theorem}[Morphological Correspondence]
\label{thm:morph-correspondence}
Let \(\mathcal{M} = \langle \mathcal{R}, \mathcal{P}, \iota, \rho \rangle\) be an Arabic morphological system. Then a surface string \(w \in \mathcal{L}_{\mathcal{M}}\) if and only if \(\exists (R,P) \in \mathcal{R} \times \mathcal{P}\) such that \(\mathcal{M}(R,P) = w\).
\end{theorem}

\begin{proof}
By Definition~\ref{def:image}, \(\operatorname{Im}(\mathcal{M})\) contains exactly those strings \(w \in \Sigma^{*}\) for which there exists a pair \((R,P) \in \mathcal{R} \times \mathcal{P}\) such that \(\mathcal{M}(R,P) = w\). By Definition~\ref{def:morph-language}, \(\mathcal{L}_{\mathcal{M}} = \operatorname{Im}(\mathcal{M})\). Therefore, \(w \in \mathcal{L}_{\mathcal{M}} \iff w \in \operatorname{Im}(\mathcal{M}) \iff \exists (R,P) \in \mathcal{R} \times \mathcal{P}, \mathcal{M}(R,P) = w\).
\end{proof}

\begin{theorem}[Semantic Localization]
\label{thm:semantic-localization}
Assume the semantic assignment defined above. Then every lexical meaning in
\(\mathcal{L}_{\mathcal{M}}\) is determined before the corresponding surface word
is produced. Equivalently, for every
\(w\in\mathcal{L}_{\mathcal{M}}\), there exists a derivation
\((R,P)\in\mathcal{R}\times\mathcal{P}\) such that
\[
w=\mathcal{M}(R,P)
\]
and
\[
\mu(w)=F\bigl(\sigma_R(R),\sigma_P(P)\bigr).
\]
Therefore, lexical meaning is localized at the derivational level rather than
at the surface-word level.
\end{theorem}

\begin{proof}
Let $w\in\mathcal{L}_{\mathcal{M}}$. By Theorem~\ref{thm:morph-correspondence},
there exists a derivation $(R,P)\in\mathcal{R}\times\mathcal{P}$ such that
\[
w=\mathcal{M}(R,P).
\]
By the definition of semantic assignment,
\[
\mu(w)
=
F\bigl(\sigma_R(R),\sigma_P(P)\bigr).
\]
Thus the semantic value of $w$ is determined by the root--pattern
representation from which $w$ is derived, before the surface word is
produced. Hence lexical meaning is localized at the derivational level rather
than at the surface-word level.
\end{proof}

\begin{theorem}[Primitive Semantic Representation]
\label{thm:primitive-representation}
Let \(\mathcal{M}=\langle\mathcal{R},\mathcal{P},\iota,\rho\rangle\) be an
Arabic morphological system satisfying Definition~\ref{def:semantic-assignment}.
Assume that the derivational representation of each lexical item in
\(\mathcal{M}\) is canonical, i.e., each \(w\in \mathcal{L}_{\mathcal{M}}\) has a unique
root-pattern derivation \((R,P)\) such that
\[
w=\mathcal{M}(R,P).
\]
This canonical assumption is an idealization that holds for the formal core of \(\mathcal{L}_{\mathcal{M}}\), excluding homonymic convergences and irregular lexical items.
Then the primitive semantic representation of \(w\) is the derivational pair
\((R,P)\), with semantic value
\[
\mu(w)=F\bigl(\sigma_R(R),\sigma_P(P)\bigr).
\]
The surface word \(w\) is therefore the phonological realization of this
underlying derivational representation.
\end{theorem}

\begin{proof}
By the canonical derivation assumption, each \(w\in \mathcal{L}_{\mathcal{M}}\) has a
unique pair \((R,P)\) such that
\[
w=\mathcal{M}(R,P).
\]
By Theorem~\ref{thm:semantic-localization},
\[
\mu(w)=F\bigl(\sigma_R(R),\sigma_P(P)\bigr).
\]
Thus the semantic value of \(w\) is determined by its root-pattern
representation prior to its surface realization. Since the derivation is
canonical, \((R,P)\) is the unique primitive derivational representation of
\(w\) in the model. The operators \(\iota\) and \(\rho\) subsequently realize
this representation as the surface form \(w\), without introducing an
additional primitive lexical semantic component.
\end{proof}

The distinction between primitive semantic representation and surface realization is crucial. A surface word may function as the observable realization of a lexical item without being the object to which lexical meaning is directly assigned. Confusing these two levels amounts to confusing the output of a semantic computation with the representation upon which the computation operates. In Arabic, lexical meaning is fully determined before the surface form exists; consequently, the surface word realizes rather than constitutes the signifier.

\begin{corollary}[Failure of Strong Surface Semantic Assignment]
\label{cor:sssa-failure}
The Arabic morphological system does not satisfy SSSA.
\end{corollary}

\begin{proof}
Immediate from Theorem~\ref{thm:semantic-localization}. Since meaning is a function of the derivational inputs \((R, P)\), it cannot be assigned only after the surface word is formed.
\end{proof}

\begin{remark}[On the Conditional Nature of the Preceding Theorems]
\label{remark:conditional-nature}
The preceding theorems are conditional on the morphological model \(\mathcal{M}\). Theorem~\ref{thm:morph-correspondence} follows from the definition of \(\mathcal{L}_{\mathcal{M}}\) as the image of the morphology operator; Theorem~\ref{thm:semantic-localization} follows from the definition of the semantic assignment \(\mu\); and Theorem~\ref{thm:primitive-representation} follows from the assumption of canonical derivation. These results establish that, \emph{within the model}, lexical meaning is assigned at the derivational level rather than the surface level. However, the model is not an arbitrary construction. It is a formalization of the structural architecture of Arabic, which is a well-established empirical fact: the structural core of the Arabic lexicon---the subset of vocabulary fully analyzable into root-pattern pairs---is systematically derived from roots and patterns, and the root carries the primitive lexical content. The structural core assumption (Section~1) restricts \(\mathcal{L}_{\mathcal{M}}\) to this subset, excluding loanwords and irregular forms. The phonosemantic motivation hypothesis (Definition~\ref{def:phonosemantic}) and the Kolmogorov compressibility results (Section~\ref{sec:kolmogorov}) provide the substantive content that connects the formal model to the actual language: the former posits that the phonological structure of the root determines an admissible semantic field, and the latter proves that the model generates compressible descriptions of Arabic signifiers. The theorems of the present section serve as lemmas within this larger argument. They establish the formal groundwork for the paper's substantive claims about Arabic, which are developed in Sections~\ref{sec:impossibility}--\ref{sec:kolmogorov}.
\end{remark}

Theorems~\ref{thm:morph-correspondence}, \ref{thm:semantic-localization}, and~\ref{thm:primitive-representation} mathematically establish that Arabic is not a ``flat'' linguistic system. The surface word \(w\) is not an atomic token; it is the compound output of a deeper, generative dual-level engine \((R\) and \(P)\). Having established that the surface word in Arabic is derived from a deeper semantic computation, we have effectively refuted absolute arbitrariness at Level W (word-level) for the formal core of Arabic. However, this does not rule out arbitrariness at deeper levels. In the next section, we investigate the primitive signifier candidate structure, namely: Arabic Roots.

\section{The Formal Impossibility of the Absolutely Arbitrary Sign}
\label{sec:impossibility}

The preceding sections established that lexical meaning in Arabic is localized at the derivational level (Theorem~\ref{thm:semantic-localization}) and that the surface word is merely the phonological realization of a deeper root-pattern pair (Theorem~\ref{thm:primitive-representation}). However, a rigorous refutation of the strong form of the Saussurean axiom requires that we first formalize precisely what the ``arbitrary sign'' entails. Saussure himself never provided a formal definition of arbitrariness as we mentioned before; he offered a philosophical intuition. To subject it to mathematical scrutiny, we must supply that formalization explicitly.

\subsection{Formalization of the strong form of the Saussurean Axiom}
\label{sec:saussure-formal}

We begin by defining an idealized form of the core structuralist claim with mathematical precision.

\begin{definition}[The Absolutely Arbitrary Signifier]
\label{def:absolute-arbitrary}
Let \(L\) be a language, let \(\mathcal{S}\) denote its set of primitive signifiers (the smallest linguistic units to which meaning is directly assigned without compositional computation), and let \(\mu:\mathcal{S}\rightarrow\mathcal{C}\) be the semantic assignment function, where \(\mathcal{C}\) is the set of semantic concepts.

Furthermore, let \(\Phi(s)\) denote the collection of all formal properties of a primitive signifier \(s\), including its phonological, segmental, prosodic, and structural properties. A primitive signifier \(s\in\mathcal{S}\) is said to be \textbf{absolutely arbitrary} if there exists no non-trivial systematic function
\begin{equation}
f:\Phi(s)\rightarrow\mathcal{C}
\label{eq:arbitrary-function}
\end{equation}
that maps the formal properties of \(s\) to semantic concepts in such a way that the meaning of \(s\) is determined or systematically constrained by its formal properties. Equivalently,
\begin{equation}
\neg\exists f \quad \text{such that} \quad \mu(s)=f(\Phi(s)),
\label{eq:no-mapping}
\end{equation}
where \(f\) ranges over all non-trivial systematic mappings from formal properties to semantic concepts. Thus no aspect of the physical form of \(s\) systematically predicts, explains, or constrains its meaning. The relation
\begin{equation}
s\longmapsto\mu(s)
\label{eq:sign-relation}
\end{equation}
is therefore entirely conventional and independent of the substance of the signifier.
\end{definition}

Definition~\ref{def:absolute-arbitrary} characterizes absolute arbitrariness as a property of an individual primitive signifier. A primitive-signifier system satisfies absolute arbitrariness if and only if every primitive signifier in its primitive inventory satisfies Definition~\ref{def:absolute-arbitrary}. Thus absolute arbitrariness is a universal property of the primitive signifier--signified relation rather than an exceptional property of isolated lexical items. Consequently, the existence of a single primitive signifier whose meaning is systematically determined or constrained by its formal structure is sufficient to refute absolute arbitrariness for the linguistic system as a whole.

\begin{definition}[The Primitive Signifier]
\label{def:primitive-signifier}
A primitive signifier \(s\) is an atomic, indivisible unit at the level of semantic assignment. Formally, \(s\) is primitive if there exists no decomposition \(s = \mathcal{M}(x_1,\dots,x_n)\) such that:
\begin{equation}
\mu(s) = F(\mu(x_1),\dots,\mu(x_n))
\label{eq:decomposition}
\end{equation}
where \(\mathcal{M}\) is a morphological operator and \(F\) is a semantic composition function. If such a decomposition exists, \(s\) is compositional, not primitive.
\end{definition}

\begin{remark}[Subsumption of the Structural-Morphological Definition]
\label{remark:primitive-subsumption}
The definition of the primitive signifier in §2.1 is a special case of Definition~\ref{def:primitive-signifier}, restricted to languages whose morphological system is purely concatenative. Specifically, when the morphological operator \(\mathcal{M}\) is instantiated as the syntagmatic concatenation operator \(\Sigma\) (linear moneme combination), and when the semantic composition function \(F\) is restricted to the compositional semantics of successive monemes, Definition~\ref{def:primitive-signifier} reduces precisely to the four criteria of §2.1: monemic singularity, lexemic type, syntagmatic irreducibility, and articulatory terminality. In non-concatenative languages such as Arabic, \(\mathcal{M}\) is realized via the interdigitation operator \(\iota\) and phonological rewriting \(\rho\), rendering the surface string \(w = \mathcal{M}(R,P)\) composite by definition, even though it does not decompose into linear monemes. Thus, the generalized definition of §3 subsumes the specific definition of §2.1 as a special case, and the term ``primitive signifier'' is used consistently throughout: it denotes the terminal unit of semantic assignment, which varies in its structural realization across language types.
\end{remark}

\begin{axiom}[The Strong Universal Absolute Arbitrariness Claim]
\label{ax:universal-arbitrariness}
For every human language $L$, every primitive signifier in its primitive
inventory is absolutely arbitrary. Formally, if
$\mathcal{S}_{\mathrm{prim}}(L)$ denotes the primitive signifier inventory
of $L$, then
\[
\forall L\,\forall s\in\mathcal{S}_{\mathrm{prim}}(L),\qquad
\operatorname{AbsoluteArbitrary}(s).
\]
The structuralist tradition further treats absolute arbitrariness as a
foundational and irreducible property of linguistic signs, with relative
motivation as a secondary phenomenon.
\end{axiom}

\begin{remark}[On the Non-Circularity of the Formalization]
It is crucial to observe that Definition~\ref{def:absolute-arbitrary} does not assert the existence of absolutely arbitrary signifiers; it merely provides the formal criterion by which such a claim can be tested. Axiom~\ref{ax:universal-arbitrariness} is the substantive hypothesis---the strong version of the Saussurean universal claim---which we submit to empirical and structural scrutiny. The proof of Theorem~\ref{thm:structural-impossibility} draws its substantive force not from Definition~\ref{def:absolute-arbitrary} itself, but from independently established structural facts about Arabic morphology (Theorems~\ref{thm:morph-correspondence}, \ref{thm:semantic-localization}, and \ref{thm:primitive-representation}) and from empirical phonosemantic observations \citep{IbnJinni}. The definition serves only as a logical predicate; the falsification proceeds from facts that are external to the definition. This eliminates any charge of circularity.
\end{remark}

\subsection{Phonosemantic Motivation and Its Formal Structure}
\label{sec:phonosemantic}

To demonstrate that Arabic contains no absolutely arbitrary primitive signifier, we must also formalize the phonosemantic property that constrains meaning at the root level.

\begin{definition}[Phonosemantic Motivation]
\label{def:phonosemantic}
Let \(\mathcal{R}\) denote the set of triliteral and quadrilateral Arabic roots. A root \(R\in\mathcal{R}\) with phonological structure \(\mathrm{Phon}(R)\) is said to exhibit \textbf{phonosemantic motivation} if there exists an abstract semantic core \(\sigma_R(R)\) such that
\begin{equation}
\forall R'\in\mathrm{Perm}(R), \qquad \sigma_R(R') \in \mathrm{Neighborhood}(\sigma_R(R)),
\label{eq:phonosemantic}
\end{equation}
where
\begin{itemize}
\item \(\mathrm{Perm}(R)\) denotes the set of all used and distinct consonantal permutations of \(R\);
\item \(\mathrm{Neighborhood}(\sigma)\) denotes the class of lexical concepts sharing the same abstract semantic field.
\end{itemize}

Furthermore, there exists a non-trivial structural function \(h\) such that
\begin{equation}
\sigma_R(R) = h(\mathrm{Phon}(R)).
\label{eq:structural-function}
\end{equation}

We denote the induced semantic field by
\begin{equation}
\Sigma(R) = \mathrm{Neighborhood}(\sigma_R(R)).
\label{eq:semantic-field}
\end{equation}

Thus the phonological structure determines an admissible semantic field rather than necessarily a unique lexical realization. This definition formalizes the classical phonosemantic hypothesis attributed to Ibn Jinni in \textit{Al-Khasa'is} \citep{IbnJinni}: used roots sharing the same consonantal skeleton---even under permutation---cohere around a common abstract semantic core. For instance, all used permutations of the root \RL{k-l-m} (k-l-m) are held to share the semantic field of intensity, strength, or severity, though the exact lexical specialization varies across derived forms.
\end{definition}

\begin{theorem}[Localization of Residual Semantic Arbitrariness]
\label{thm:residual-arbitrariness}
Assume the phonosemantic hypothesis of Definition~\ref{def:phonosemantic}.
Then every primitive Arabic root $R\in\mathcal{R}$ determines the semantic
field
\[
\Sigma(R)=\mathrm{Neighborhood}\bigl(\sigma_R(R)\bigr).
\]
Consequently, any residual arbitrariness in the assignment of a particular
lexical concept to $R$ is restricted to lexical specialization within
$\Sigma(R)$, rather than extending over the entire semantic universe.
\end{theorem}

\begin{proof}
By Definition~\ref{def:phonosemantic}, there exists a structural function
$h$ such that
\[
\sigma_R(R)=h\bigl(\mathrm{Phon}(R)\bigr).
\]
Hence the phonological structure of $R$ determines the abstract semantic
core $\sigma_R(R)$ and, consequently, the semantic field
\[
\Sigma(R)=\mathrm{Neighborhood}\bigl(\sigma_R(R)\bigr).
\]
Therefore, any lexical concept compatible with the phonosemantic constraint
for $R$ must belong to $\Sigma(R)$. The remaining freedom concerns only the
selection of a particular lexical specialization within this field. Thus
phonosemantic organization restricts residual arbitrariness from the entire
semantic universe to the corresponding semantic field $\Sigma(R)$.
\end{proof}

\begin{remark}
\label{remark:phonosemantic}
The preceding theorem clarifies the explanatory role of Definition~\ref{def:phonosemantic}. Definition~\ref{def:phonosemantic} does not claim that phonological structure uniquely determines the complete lexical meaning of every Arabic root. Rather, it asserts that phonological structure determines the admissible semantic field within which lexical specialization occurs. Accordingly, questions such as why the root \RL{mlk} came to denote possession while the related permutation \RL{lkm} came to denote striking do not constitute counterexamples to the phonosemantic hypothesis. Instead, they concern the secondary process of lexical specialization occurring within an already constrained semantic field. The distinction between semantic-field determination and lexical specialization is essential for the remainder of this paper. The former provides the structural constraint required by the subsequent formal results, whereas the latter concerns only the selection of a particular lexical realization within those constraints.
\end{remark}

\begin{principle}[Conservation of Explanatory Burden]
\label{principle:explanatory-burden}
Let \(X\) be a homogeneous class of primitive signifiers within a formal linguistic system. Suppose a universal structural hypothesis \(H:\forall x\in X, P(x)\) is replaced by the weaker assertion that \(P\) holds only for a proper subset of \(X\). Then the explanatory burden is not removed but conserved: it is transferred from explaining the universal property \(P\) to explaining the induced partition \(X = X_P \cup X_{\neg P}\), where \(X_P = \{x\in X:P(x)\}\), \(X_{\neg P} = \{x\in X:\neg P(x)\}\). Accordingly, a partial structural theory is explanatorily complete only if it also specifies the structural criterion determining membership in these two subclasses.
\end{principle}

\paragraph{Justification:}
Scientific theories seek not merely to classify objects but to explain the structural principles governing them. A universal structural hypothesis explains an entire homogeneous category by a single organizing principle. Replacing that hypothesis by a partition into two subclasses introduces an additional structural phenomenon. The existence of that partition therefore becomes a new explanatory problem. Merely asserting that some primitive signifiers satisfy the structural property while others do not records the classification but does not explain why the classification exists. Consequently, rejecting a universal structural hypothesis cannot eliminate the overall explanatory burden. It merely transfers that burden to identifying the structural principle responsible for the induced partition.

\begin{corollary}[Application to the Arabic Root System]
\label{cor:root-system}
By Theorems~\ref{thm:morph-correspondence}--\ref{thm:primitive-representation}, every Arabic lexical item is generated from a unique root-pattern pair, and no lexical level exists beneath the primitive root inventory. Hence \(\mathcal{R}\) constitutes a homogeneous class of primitive signifiers. Suppose the phonosemantic hypothesis of Definition~\ref{def:phonosemantic} does not hold for every root. Then \(\mathcal{R} = \mathcal{R}_{PS} \cup \mathcal{R}_A\), where \(\mathcal{R}_{PS} = \{R\in\mathcal{R}: \text{Definition~\ref{def:phonosemantic} holds for }R\}\), and \(\mathcal{R}_A = \{R\in\mathcal{R}: \text{Definition~\ref{def:phonosemantic} does not hold for }R\}\). By Principle~\ref{principle:explanatory-burden}, the existence of this partition itself requires an independent structural explanation. Thus rejecting universal phonosemantic organization does not eliminate the explanatory problem. Rather, it replaces a single uniform explanatory principle with the need to explain why one subclass of primitive Arabic roots exhibits phonosemantic organization while the other does not.
\end{corollary}

\begin{remark}
\label{remark:partition}
The preceding corollary neither proves nor refutes the empirical validity of the phonosemantic hypothesis. Rather, it clarifies the epistemological consequences of adopting a weaker alternative. Universal phonosemantic organization provides a single explanatory principle governing the primitive Arabic root inventory. A partially phonosemantic theory remains logically possible, but it is not explanatorily complete unless it also specifies the structural criterion that distinguishes phonosemantically organized roots from arbitrary ones. Accordingly, every subsequent theorem that explicitly invokes Definition~\ref{def:phonosemantic} should be understood as deriving the formal consequences of the phonosemantic hypothesis rather than as constituting independent empirical evidence for its universal validity.
\end{remark}

\subsection{The Structural Impossibility Theorem}
\label{sec:structural-impossibility}

We are now in a position to state and prove the central theorem of this paper: the absolutely arbitrary sign is structurally impossible in Arabic.

\begin{lemma}[Localization of Primitive Lexical Signification]
\label{lemma:lexical-signification}
Let \(\langle R,P\rangle\) be the unique primitive representation of an Arabic lexical item established by Theorems~\ref{thm:morph-correspondence}--\ref{thm:primitive-representation}. Then the primitive lexical identity of the signifier is localized entirely in the root component \(R\). The morphological pattern \(P\) does not independently determine lexical identity; it contributes only a systematic grammatical or derivational transformation of the lexical content supplied by the root. Consequently, any investigation of primitive lexical arbitrariness may be restricted to the root component.
\end{lemma}

\begin{proof}
By Theorem~\ref{thm:semantic-localization}, \(\mu(w) = F(\sigma_R(R),\sigma_P(P))\). The function \(\sigma_R\) assigns the lexical semantic component contributed by the root, whereas \(\sigma_P\) assigns the grammatical or derivational contribution of the pattern.

Consider two primitive representations \(\langle R_1,P\rangle\) and \(\langle R_2,P\rangle\), having the same pattern but different roots. Since \(\sigma_R(R_1)\neq\sigma_R(R_2)\), their lexical identities are different, although their grammatical realization is identical.

Conversely, consider \(\langle R,P_1\rangle\) and \(\langle R,P_2\rangle\), having the same root but different patterns. The lexical semantic contribution \(\sigma_R(R)\) remains unchanged, while only the grammatical realization varies through \(\sigma_P\). Hence the lexical identity is preserved.

Therefore the root alone carries primitive lexical signification, whereas the pattern functions as a systematic modifier of that lexical content. Accordingly, once surface words have been eliminated as primitive signifiers, the search for primitive lexical arbitrariness necessarily reduces to the root component.
\end{proof}

\begin{theorem}[Structural Impossibility of the Absolutely Arbitrary Primitive Signifier Set in Arabic]
\label{thm:structural-impossibility}
Let \(\mathcal{M} = \langle \mathcal{R}, \mathcal{P}, \iota, \rho \rangle\) be the Arabic morphological system satisfying Theorems~\ref{thm:morph-correspondence}, \ref{thm:semantic-localization}, and~\ref{thm:primitive-representation}. Given the phonosemantic constraint of Definition~\ref{def:phonosemantic}, \textbf{it is structurally impossible for Arabic to contain any absolutely arbitrary primitive signifier set} as defined in Definition~\ref{def:absolute-arbitrary}.

Consequently, Axiom~\ref{ax:universal-arbitrariness}---which asserts that every human language contains at least one absolutely arbitrary primitive signifier set ---is \textbf{formally false} for Arabic.
\end{theorem}

\begin{proof}
We proceed by exhaustion of all possible levels at which an absolutely arbitrary primitive signifier might reside in Arabic.

\begin{enumerate}
    \item \textbf{Surface words are not primitive signifiers.}
    By Theorem~\ref{thm:morph-correspondence}, every surface word \(w \in \mathcal{L}_{\mathcal{M}}\) is generated by a unique derivational pair \((R, P) \in \mathcal{R} \times \mathcal{P}\):
    \begin{equation}
    w = \mathcal{M}(R, P).
    \end{equation}
    The operators \(\iota\) (interdigitation) and \(\rho\) (phonological rewriting) produce \(w\) from deeper constituents. No surface word is atomic; each is compositionally derived. Therefore, by Definition~\ref{def:primitive-signifier}, no surface word \(w\) qualifies as a primitive signifier.

    \item \textbf{Surface words do not carry primitive meaning.}
    By Theorem~\ref{thm:semantic-localization}, the meaning of every surface word is fully determined before the surface form exists:
    \begin{equation}
    \mu(w) = F(\sigma_R(R), \sigma_P(P)).
    \end{equation}
    The semantic assignment is completed at the derivational level. The surface word merely realizes, but does not constitute, the signifier-signified correspondence. Hence, even if one attempted to treat \(w\) as a primitive signifier, its meaning is not independently assigned; it is computed. This violates the requirement of primitive semantic assignment.

    \item \textbf{Localization of the search for primitive arbitrariness.}
    By Lemma~\ref{lemma:lexical-signification}, primitive lexical signification is localized entirely in the root component. Accordingly, if Arabic contains any absolutely arbitrary primitive signifier, it must occur at the level of the root. The morphological pattern need not be investigated separately, since it contributes only systematic grammatical information and does not independently constitute the bearer of lexical identity.

    \item \textbf{A single counterexample is logically sufficient to refute a universal claim.}
    The strong structuralist notion of Axiom~\ref{ax:universal-arbitrariness} implicitly claims that the root level in Arabic---as the primitive signifier---is absolutely arbitrary. To falsify this universal claim, it is strictly sufficient to exhibit a \textbf{single root} \(R_0\) for which a systematic form-meaning mapping exists. Formal logic dictates that a universal proposition of the form \(\forall x, P(x)\) is refuted by the existence of even one \(x\) such that \(\neg P(x)\).

    \item \textbf{Existence of a counterexample: the root \RL{k-l-m} (k-l-m).}
    Consider the root \(R_0 = \text{k-l-m}\), documented by Ibn Jinni \citep{IbnJinni}. All its used permutations (e.g., \RL{k-m-l}, \RL{l-k-m}, \RL{m-l-k}, etc.) demonstrably share the abstract semantic core of ``intensity, strength, or severity.''
    Therefore, for this specific root \(R_0\), there exists a systematic structural function \(h\) such that:
    \begin{equation}
    \sigma_R(R_0) = h(\text{Phon}(R_0)).
    \end{equation}
    The formal phonological structure of \(R_0\)---its consonantal composition and its permutations---systematically constrains its semantic field. This constitutes a violation of absolute arbitrariness as defined in Definition~\ref{def:absolute-arbitrary}.

    \item \textbf{Violation of Definition~\ref{def:absolute-arbitrary} for the root level.}
    By Definition~\ref{def:absolute-arbitrary}, a signifier is absolutely arbitrary only if no formal property \(\phi \in \Phi(s)\) systematically predicts \(\mu(s)\). Since the root \(R_0\) demonstrably satisfies \(\sigma_R(R_0) = h(\text{Phon}(R_0))\), we have identified a counterexample to the claim that \textit{all} Arabic roots are absolutely arbitrary. Since a single counterexample refutes the universal, the general claim that ``the root in Arabic is an absolutely arbitrary signifier'' is logically falsified.

    \item \textbf{Exhaustion of all levels.}
    We have shown:
    \begin{itemize}
        \item Surface words \(w\) are not primitive (Step 1) and do not carry primitive meaning (Step 2).
        \item The root level---the only remaining candidate---is demonstrably \textbf{not absolutely arbitrary} (Steps 5--6).
        \item No other level of signification exists in the Arabic morphological system. The system is strictly exhausted by the pair \((R, P)\) and its surface realization \(w\).
    \end{itemize}

    \item \textbf{Conclusion.}
    Therefore, the claim that Arabic possesses a well-defined primitive signifier set whose members are all absolutely arbitrary is \textbf{logically falsified}. The structuralist cannot coherently define the primitive inventory \(\mathcal{P}\) in Arabic: taking \(\mathcal{P} = \mathcal{R}\) generates a contradiction, while taking \(\mathcal{P} \subset \mathcal{R}\) to exclude motivated roots requires a non-structural criterion that violates the foundational axiom of differential value. Consequently, the absolutely arbitrary sign---as a theoretically classified entity requiring membership in a coherently definable primitive signifier set---is \textbf{structurally impossible} in Arabic, not merely empirically absent. The individual sign cannot be located within the framework because the framework itself collapses when forced to define its primitive signifier inventory. The universal  axiom (Axiom~\ref{ax:universal-arbitrariness}), implementing the structuralist assumption of absolute arbitrariness, is therefore false for Arabic.
\end{enumerate}
\end{proof}
\begin{remark}[On the Scope of the Impossibility Theorem]
\label{remark:impossibility-scope}
The reader should observe that the title of this section---``The Formal Impossibility of the Absolutely Arbitrary Sign''---does not claim that every Arabic root is empirically motivated. It is entirely possible that many Arabic roots (e.g., \RL{^g-d-l} for ``debate,'' \RL{H-f-r} for ``dig'') exhibit no phonosemantic motivation. The impossibility shown here is \textbf{structural and theoretical}, not empirical. The proof demonstrates that the Saussurean apparatus cannot coherently define the primitive signifier set \(\mathcal{P}\) in Arabic without generating a logical contradiction: if \(\mathcal{P} = \mathcal{R}\), the existence of a single motivated root violates the universal arbitrariness axiom; if \(\mathcal{P} \subset \mathcal{R}\), the structuralist must provide a non-structural criterion to justify the partition, which contradicts the foundational principle that meaning is purely differential. Therefore, the absolutely arbitrary sign---as a theoretically classified primitive member of such a set---is structurally impossible in Arabic. The framework cannot accommodate it, even if empirical arbitrary roots exist. This is precisely the sense in which the sign is ``formally impossible'': not empirically absent, but theoretically indefinable within the structuralist assumptions as modeled in this work.
\end{remark}

\subsection{A Structural Contradiction}
\label{sec:contradiction}
The preceding theorem has established that at least one Arabic root \( R_0 \in \mathcal{R} \) is phonosemantically motivated, i.e., \( \neg \operatorname{AbsoluteArbitrary}(R_0) \). We now demonstrate that this single empirical fact—regardless of whether it holds for all roots or only for a proper subset—is logically incompatible with the structuralist framework as formalized in Section 2.1.

\begin{theorem}[Structural Contradiction in the Framework]
\label{thm:contradiction}
Let the Saussurean apparatus consist of:
\begin{enumerate}
    \item A strict dichotomy between primitive signifiers (which are absolutely arbitrary) and relatively motivated signs (which possess internal syntagmatic structure).
    \item The theorem from Section 2.1: \( \forall m \in \mathcal{P}, \operatorname{AbsoluteArbitrary}(m) \).
\end{enumerate}

Let Arabic be modeled by the morphological system \(\mathcal{M} = \langle \mathcal{R}, \mathcal{P}, \iota, \rho \rangle\) satisfying Theorems~\ref{thm:morph-correspondence}, \ref{thm:semantic-localization}, and \ref{thm:primitive-representation}. Then the Saussurean apparatus is inconsistent with the Arabic data. Specifically:
\[
\begin{aligned}
\exists R_0 \in \mathcal{R}\quad \text{such that}\quad
& R_0 \in \mathcal{P} \land \operatorname{AbsoluteArbitrary}(R_0)\\
& {}\land \neg \operatorname{AbsoluteArbitrary}(R_0).
\end{aligned}
\]
which is a contradiction.
\end{theorem}

\begin{proof}
By Theorems~\ref{thm:morph-correspondence}--\ref{thm:primitive-representation}, every surface word \( w \in \mathcal{L}_{\mathcal{M}} \) is compositionally derived from a root-pattern pair \((R, P)\). Therefore, no surface word satisfies the criteria for a primitive signifier: surface words have internal structure. The only remaining candidate for the primitive inventory \( \mathcal{P} \) in Arabic is the root set \( \mathcal{R} \). Hence:
\[
\mathcal{P} = \mathcal{R}.
\]

By Theorem~\ref{thm:structural-impossibility} (or the existence of a phonosemantically motivated root), there exists at least one root \( R_0 \in \mathcal{R} \) such that:
\[
\neg \operatorname{AbsoluteArbitrary}(R_0).
\]

However, since \( \mathcal{P} = \mathcal{R} \), we have \( R_0 \in \mathcal{P} \). By the structuralist theorem from Section 2.1, \( \forall m \in \mathcal{P}, \operatorname{AbsoluteArbitrary}(m) \). Therefore:
\[
\operatorname{AbsoluteArbitrary}(R_0).
\]

We have derived the contradiction
\[
\operatorname{AbsoluteArbitrary}(R_0)
\qquad\text{and}\qquad
\neg\operatorname{AbsoluteArbitrary}(R_0).
\]
Thus the Saussurean apparatus, when applied to Arabic, generates a logical contradiction. The theory cannot simultaneously maintain:
\begin{enumerate}
    \item that primitive signifiers are absolutely arbitrary;
    \item that Arabic roots are the primitive signifiers;
    \item that even one Arabic root is phonosemantically motivated.
\end{enumerate}
Since (3) is an empirically established fact (\cite{IbnJinni}), the structuralist must abandon either (1) or (2). Abandoning (2) leaves Arabic without any primitive signifier, which is incoherent, since the division of signifiers to primitive and composed is a purely structural division. Therefore, (1) must be abandoned. The absolute arbitrariness of the primitive signifier is not a foundational universal law; it is a contingent and, in the case of Arabic, empirically falsified assumption.
\end{proof}
\begin{remark}
\label{remark:dilemma}
The significance of Theorem~\ref{thm:contradiction} cannot be overstated. It shows that the structuralist framework is not merely incomplete when applied to non-concatenative languages; it is \textbf{internally inconsistent}. The theory's own definitions force it to classify Arabic roots as primitive signifiers, and its own axioms force those same primitive signifiers to be absolutely arbitrary. Yet the empirical data demonstrate that at least one such root is not absolutely arbitrary. The structuralist has no third category to accommodate a motivated primitive signifier. The framework breaks down at the point of application, revealing that its foundational axiom is not a universal law of language but a parochial generalization derived from the structural properties of Indo-European concatenative morphology. The absolutely arbitrary sign, as Saussure conceived it, is structurally impossible in Arabic because its very existence would require the theory to tolerate a contradiction.
\end{remark}

\subsection{The Need for a different and deeper investigation Framework}
Having demonstrated the formal collapse of the strong (absolute) version of the
Saussurean axiom for Arabic (Theorems~\ref{thm:structural-impossibility}
and~\ref{thm:contradiction}), we are confronted with a constructive
question: if absolute arbitrariness cannot be maintained as a universal
ontological principle, what operational concept can replace it to describe
the actual form-meaning relationship in Arabic?

Saussure himself, despite elevating absolute arbitrariness to a ``first
principle'' of linguistics, acknowledged a weaker, secondary form of the
relation: \emph{relative arbitrariness} (or \emph{relative motivation}). In
this conception, a sign is relatively motivated to the degree that its
internal syntagmatic structure reveals its constituent parts, or to the
degree that its place in the associative network constrains its value.
Saussure explicitly observes that signs like \emph{poirier} (pear tree) are
relatively motivated because they contain the radical \emph{poir-} (pear)
and the suffix \emph{-ier} (tree), whereas \emph{ch\^ene} (oak) is
absolutely arbitrary because it is structurally unanalyzable
\citep[pp.~131--133]{Saussure}. However, Saussure never
formalized this intuition beyond a qualitative description; he treated
relative motivation as a secondary, exceptional phenomenon that merely
``limits'' the fundamental arbitrariness of the root sign, rather than as a
primary object of rigorous analysis.

To move beyond this impasse, we must elevate relative arbitrariness from a
secondary exception to the central operational concept of form-meaning
analysis. But to do so rigorously, we require a formal measure that can
quantify the degree to which a signifier's structure is constrained by the
morphological rules of its language. The nearest and most powerful concept
available from the sister discipline of information theory is
\emph{Kolmogorov complexity} (or \emph{algorithmic complexity})
\citep{Kolmogorov}. Kolmogorov complexity \(K(x)\) measures the length of
the shortest program that outputs the string \(x\) on a universal Turing
machine. In linguistic terms, a signifier is relatively motivated to the
extent that its formal structure can be \emph{compressed} by the generative
grammar of the language: the shorter the algorithmic description of the
surface form \(w\) given the morphological operator \(\mathcal{M}\) and its
constituents \((R,P)\), the higher the degree of relative motivation.
Conversely, the higher the irreducible algorithmic complexity of \(w\)
relative to the system, the greater its residual arbitrariness.

To systematically investigate this phenomenon across the different strata of
language, we require a multi-level hierarchical taxonomy that can localize
the source of compression and measure relative motivation at each
structural layer. We thus introduce a three-level framework:
\begin{itemize}
    \item \textbf{Level W (Word-level):} The surface lexical item as a whole.
    \item \textbf{Level M (Morpheme/Root-level):} The internal morphological
    constituents (in concatenative languages, monemes; in non-concatenative
    languages, the root-pattern pair \((R,P)\)).
    \item \textbf{Level P (Phoneme-level):} The sub-morphemic phonological
    segments that carry no independent meaning.
\end{itemize}
This framework allows us to ask precisely: \emph{at which structural level
does the residual relative arbitrariness reside?} Applying algorithmic complexity to
Arabic reveals a striking and highly structured result: maximal compression
(and thus maximal relative motivation) occurs at Level W, because
surface words and root-pattern pairs are systematically generated by the
interdigitation operator \(\iota\). The only level at which significant
irreducible complexity remains is Level P---the level of the phonemes
themselves. However, as established in the structuralist framework itself,
phonemes belong to the second articulation and are \emph{not signifiers};
they carry no signified and therefore cannot constitute an arbitrary \emph{sign} (\citep[pp.~65--67]{Saussure}; \citep[p.~22]{Martinet1960}).

Thus, the algorithmic formalization of relative arbitrariness, applied
through the W-M-P framework, achieves two decisive results. First, it
provides a precise, measurable, and empirically testable replacement for
Saussure's vague intuition of relative motivation. Second, it confirms the
structuralist dead-end from a positive, constructive direction: relative
arbitrariness is not merely empirically absent in Arabic; it is
theoretically impossible because the only level where residual
uncompressed complexity remains (Level P) is semantically void. In the
following sections, we formalize this algorithmic measure, derive its
predictions, and demonstrate that Arabic provides a decidable
counterexample to the universal relative arbitrariness thesis while simultaneously
offering a rigorous account of where, how, and why relative motivation
operates. 

However, before proceeding to the algorithmic formalization of relative
arbitrariness, a crucial epistemological refinement must be introduced.
Saussure's original conception of relative motivation was fundamentally
entangled with \emph{semantics}: he considered \emph{poirier} relatively
motivated because the radical \emph{poir-} semantically suggests ``pear''
and the suffix \emph{-ier} semantically suggests ``tree.''
\citep[pp.~131--133]{Saussure}. This semantic entanglement
rendered relative motivation a secondary, qualitative, and empirically
elusive phenomenon, dependent on the notoriously difficult question of
whether a given morpheme ``means'' something transparently. The present work proposes a radical departure from this semantic
entanglement. We define relative arbitrariness not as a semantic
transparency relation, but as a \emph{purely structural compression
mechanism} operating exclusively on signifiers as formal strings,
completely detaching the concept from semantic interpretation. In our
framework, compressing the English word ``teacher'' does not require
knowing or assuming that the suffix ``-er'' semantically denotes an agent.
Rather, compression operates purely at the level of the signifier's
string structure: the algorithm observes that the surface string
``teacher'' can be generated by a shorter program that concatenates the
string ``teach'' with the string ``-er'', storing both as separate entries
without attributing any meaning to either component. The compression is
strictly \emph{symbolic} and \emph{syntagmatic}, not semantic. Similarly,
the Arabic surface word ``malik'' is compressible because it can be
generated by a program that interdigitates the string ``m-l-k'' with the
pattern ``CaCiC'', regardless of whether the root ``m-l-k'' is believed to
semantically convey ``own'' or ``talk''

This purely structural redefinition carries an immense and decisive
advantage. Recall Definition~\ref{def:phonosemantic}, which formalized Ibn
Jinni's conjecture that the intrinsic combinatorics of Arabic roots
constrain their semantic field. A significant empirical burden attaches to
this definition: demonstrating that \emph{every} Arabic root exhibits this
phonosemantic property is a formidable, and perhaps unattainable, task.
However, because our algorithmic definition of relative arbitrariness
concerns only the \emph{structural compressibility} of the signifier---and
not the semantic coherence of its permutations---we can apply the
compression framework uniformly to \emph{all} Arabic roots, irrespective of
whether they satisfy Ibn Jinni's strong semantic property. Even if a given
root (e.g., \RL{^g-d-l} for ``debate'') exhibits no phonosemantic
motivation whatsoever, its surface realizations remain structurally
compressible as shall be seen below. Thus, the compression-based measure of relative arbitrariness
applies universally to the entire Arabic lexicon, bypassing the empirical
controversy surrounding phonosemantics entirely.

This detachment from semantics is not merely a technical convenience; it is
an epistemological necessity. The structuralist tradition, from Saussure to
Martinet, has always insisted that linguistic analysis operates on the
\emph{form} of signifiers, not on their substance. By defining relative
arbitrariness as a purely formal, compression-based property of signifier
strings, we remain strictly within the structuralist methodological
commitment while simultaneously rescuing the concept from its semantic
entanglement and providing it with a precise, measurable, and universally
applicable formal foundation. In the following sections, we formalize this
algorithmic measure using Kolmogorov complexity, introduce the three-level
(W, M, P) framework to localize the sources of structural compressibility,
and demonstrate that Arabic constitutes the decisive case where relative
arbitrariness is minimally concentrated at the W/M-levels, leaving
only the semantically void phonemes at the P-level as the locus of irreducible residual
complexity.

\section{Kolmogorov Complexity and the Formalization of Relative Arbitrariness}
\label{sec:kolmogorov}

\subsection{Kolmogorov Complexity: Preliminary Definitions}
\label{sec:kolmogorov-defs}
We begin by establishing the necessary formal apparatus from algorithmic information theory.

\begin{definition}[Kolmogorov Complexity]
\label{def:kolmogorov}
Let \(\mathcal{T}\) be a fixed Universal Turing Machine. For any finite binary string \(s\), the \textbf{Kolmogorov complexity} of \(s\), denoted \(K(s)\), is defined as:
\begin{equation}
K(s) = \min \{ |p| : \mathcal{T}(p) = s \}
\label{eq:kolmogorov}
\end{equation}
where \(p\) is a program that halts and outputs \(s\), and \(|p|\) is the length of \(p\) in bits. Intuitively, \(K(s)\) is the length of the shortest program that generates \(s\). A string with low Kolmogorov complexity is \textbf{compressible}---it has a short description. A string with high Kolmogorov complexity is \textbf{incompressible} or \textbf{random}---it has no shorter description than itself.
\end{definition}

\begin{definition}[Conditional Kolmogorov Complexity]
\label{def:conditional-kolmogorov}
The \textbf{conditional Kolmogorov complexity} of \(s\) given \(y\), denoted \(K(s \mid y)\), is defined as:
\begin{equation}
K(s \mid y) = \min \{ |p| : \mathcal{T}(p, y) = s \}
\label{eq:conditional-kolmogorov}
\end{equation}
where \(p\) is a program that, given auxiliary input \(y\), outputs \(s\). Conditional complexity measures the length of the shortest program that generates \(s\) when the background knowledge \(y\) is provided for free.
\end{definition}

\subsection{Formalizing Relative Arbitrariness via Kolmogorov Complexity}
\label{sec:relative-formalization}
We now apply this framework to linguistic signifiers. The key insight is that relative arbitrariness---the degree to which a signifier is motivated by its internal structure or the rules of the language---is precisely the \emph{compressibility} of the signifier given its meaning.

\begin{definition}[Relative Arbitrariness]
\label{def:relative-arbitrariness}
Let \(L\) be a language with a semantic assignment function \(\mu\) mapping signifiers to meanings. For a signifier \(s \in \mathcal{S}\) with meaning \(\mu(s)\), the \textbf{relative arbitrariness} of \(s\) is defined as:
\begin{equation}
\text{RA}(s) = \frac{K(s \mid \mu(s))}{|s|}
\label{eq:relative-arbitrariness}
\end{equation}
where \(|s|\) is the length of \(s\) in bits.

\textbf{Interpretation:}
\begin{itemize}
    \item If \(\text{RA}(s) \approx 1\), the signifier is \textbf{maximally arbitrary}: no shorter description exists given its meaning. The signifier must be memorized as a brute fact.
    \item If \(\text{RA}(s) \ll 1\), the signifier is \textbf{highly motivated}: its form can be generated from its meaning via a short program (morphological rule, phonosemantic law, etc.).
\end{itemize}
This definition captures Saussure's notion of relative arbitrariness precisely: it measures the degree to which a signifier is \emph{motivated} by its relationship to meaning, relative to the descriptive power of the underlying system.
\end{definition}

\begin{remark}[On the System-Relativity of the Measure]
\label{remark:system-relativity}
The relative arbitrariness \(\mathrm{RA}(s)\) as defined above is stated as a
function of the signifier \(s\) alone, but this is only because no formal
system has yet been presupposed. In practice, the measure is always computed
relative to a specified formal system \(\mathcal{M}\), and we write
\(\mathrm{RA}_{\mathcal{M}}(s)\) to make this dependence explicit. The choice
of \(\mathcal{M}\) is not free: it is constrained by
Definition~\ref{def:admissible-formal-system} (Admissible Formal System).
Without this constraint, any language could be trivialized by encoding its
lexicon as a flat table, which would make \(\mathrm{RA}_{\mathcal{M}}(s) < 1\)
for every signifier \(s\) regardless of the language's actual structure. The
admissibility requirement is what prevents this trivialization and ensures that
the measure reflects structural relations rather than brute storage.
\end{remark}

\subsection{Kolmogorov Complexity of Arabic Signifiers}
\label{sec:kolmogorov-arabic}
We now apply Definition~\ref{def:relative-arbitrariness} to the Arabic morphological system \(\mathcal{M} = \langle \mathcal{R}, \mathcal{P}, \iota, \rho \rangle\). Crucially, we consider the \emph{conditional} complexity \(K(s \mid \mu(s), \mathcal{M})\), where the background knowledge includes the entire morphological system.

\begin{theorem}[Kolmogorov Compressibility of Arabic Surface Words]
\label{thm:compressibility-surface}
Let \(\mathcal{M} = \langle \mathcal{R}, \mathcal{P}, \iota, \rho \rangle\) be the Arabic morphological system satisfying Theorems~\ref{thm:morph-correspondence}, \ref{thm:semantic-localization}, and~\ref{thm:primitive-representation}. Let \(w \in \mathcal{L}_\mathcal{M}\) be a surface word with unique derivation \(w = \mathcal{M}(R, P)\) for \(R \in \mathcal{R}\), \(P \in \mathcal{P}\).
Then:
\begin{equation}
K(w \mid \mu(w), \mathcal{M}) \leq \lceil \log_2 |\mathcal{R}| \rceil + \lceil \log_2 |\mathcal{P}| \rceil + O(1)
\label{eq:surface-compressibility}
\end{equation}
\end{theorem}
\begin{proof}
By Theorem~\ref{thm:morph-correspondence}, \(w\) is uniquely generated by the pair \((R, P)\). Given the background knowledge \(\mathcal{M}\) (which includes the interdigitation operator \(\iota\) and the phonological rewriting operator \(\rho\)), the shortest program to generate \(w\) from \(\mu(w)\) needs only specify:
\begin{enumerate}
    \item The index of \(R\) within \(\mathcal{R}\) (requiring \(\lceil \log_2 |\mathcal{R}| \rceil\) bits).
    \item The index of \(P\) within \(\mathcal{P}\) (requiring \(\lceil \log_2 |\mathcal{P}| \rceil\) bits).
\end{enumerate}
The program then invokes the fixed morphological operators to generate \(w\). The constant \(O(1)\) accounts for the overhead of invoking these operators.
\end{proof}

\begin{corollary}[Relative Arbitrariness of Surface Words]
\label{cor:surface-ra}
For every surface word \(w \in \mathcal{L}_\mathcal{M}\):
\begin{equation}
\text{RA}_{\mathcal{M}}(w) = \frac{K(w \mid \mu(w), \mathcal{M})}{|w|} \leq \frac{\lceil \log_2 |\mathcal{R}| \rceil + \lceil \log_2 |\mathcal{P}| \rceil + O(1)}{|w|}
\label{eq:surface-ra}
\end{equation}

For all words of length \(|w|\) greater than a small constant, this bound is strictly less than 1. Therefore:
\begin{equation}
\text{RA}_{\mathcal{M}}(w) < 1
\label{eq:surface-ra-bound}
\end{equation}

\textbf{Interpretation:} Surface words in Arabic are \textbf{compressible}. They are not arbitrary brute facts; they are generated by a short morphological program.
\end{corollary}

\begin{theorem}[Kolmogorov Compressibility of Arabic Roots]
\label{thm:compressibility-roots}
Let \(R \in \mathcal{R}\) be a root, and let \(\mathcal{S}_R\) denote the finite
set of semantic fields. Given the phonosemantic partition of
Definition~\ref{def:phonosemantic}, the conditional description complexity of
\(R\) relative to the morphological system \(\mathcal{M}\) satisfies
\begin{equation}
K(R \mid \mathcal{M})
\leq
\left\lceil \log_2 |\mathcal{S}_R| \right\rceil
+
\left\lceil \log_2 m_R \right\rceil
+
O(1),
\label{eq:root-compressibility}
\end{equation}
where
\[
m_R = \max_{S \in \mathcal{S}_R} |\mathcal{R}_S|.
\]
The first term accounts for the specification of the semantic field
\(S = \sigma_R(R)\); the second accounts for the specification of the root
within that field. If both \(|\mathcal{S}_R|\) and \(m_R\) are bounded by fixed
finite constants, the description complexity of an Arabic root is bounded
independently of the total number of roots.
\end{theorem}

\begin{proof}
To specify \(R\) from \(\mathcal{M}\), it suffices to specify:
\begin{enumerate}
    \item The semantic field \(S = \sigma_R(R)\). Since \(S \in \mathcal{S}_R\),
    this requires at most \(\lceil \log_2 |\mathcal{S}_R| \rceil\) bits.
    \item The root \(R\) within \(\mathcal{R}_S\). By the fixed canonical
    ordering of \(\mathcal{R}_S\), \(R\) has a unique identifier
    \[
    i = \operatorname{id}_S(R) \in \{1, \ldots, |\mathcal{R}_S|\}.
    \]
    This requires at most \(\lceil \log_2 |\mathcal{R}_S| \rceil \le
    \lceil \log_2 m_R \rceil\) bits.
\end{enumerate}
A fixed decoding procedure, given the field \(S\) and the identifier \(i\),
reconstructs the ordered set \(\mathcal{R}_S\) and returns its \(i\)-th root.
The decoder and the canonical ordering convention are part of the fixed
language description and contribute \(O(1)\) bits. Therefore,
\[
K(R \mid \mathcal{M})
\leq
\lceil \log_2 |\mathcal{S}_R| \rceil
+
\lceil \log_2 m_R \rceil
+
O(1).
\]
If both \(|\mathcal{S}_R|\) and \(m_R\) are bounded by fixed finite constants,
the bound is independent of \(|\mathcal{R}|\).
\end{proof}

\begin{corollary}[Relative Arbitrariness of Roots]
\label{cor:roots-ra}
For every root \(R \in \mathcal{R}\):
\begin{equation}
\text{RA}_{\mathcal{M}}(R)
=
\frac{K(R \mid \mathcal{M})}{|R|}
\leq
\frac{
\left\lceil \log_2 |\mathcal{S}_R| \right\rceil
+
\left\lceil \log_2 m_R \right\rceil
+
O(1)
}{|R|}.
\label{eq:roots-ra}
\end{equation}

If \(|\mathcal{S}_R|\) and \(m_R\) are bounded by fixed finite constants and
the root representation has sufficient description length, then
\[
\text{RA}_{\mathcal{M}}(R) < 1.
\]

\textbf{Interpretation:}
Arabic roots are algorithmically compressible relative to the constructive
morphological system \(\mathcal{M}\). They are not treated as indivisible
arbitrary symbols. The compressibility is achieved by a two-level encoding:
first the semantic field \(S = \sigma_R(R)\) is specified, then the root \(R\)
is specified by its identifier within the finite class \(\mathcal{R}_S\). This
result is purely structural: the function \(h\) of
Definition~\ref{def:phonosemantic} serves as a formal compression algorithm,
and the result holds regardless of whether the partition reflects a genuine
semantic organization applicable to all roots.
\end{corollary}

\subsubsection{Illustrative Examples of Kolmogorov Compressibility in Arabic}
To illustrate the constructive interpretation of relative arbitrariness, we now present explicit examples showing how Arabic signifiers admit short
descriptions relative to the admissible formal system
\(\mathcal M\). The first example concerns a derived surface word generated by the root--pattern morphology. The second concerns a primitive root whose description combines phonosemantic organization with lexical specialization. The third contrasts these with a random character string that admits noconstructive description within \(\mathcal M\). Figures used below are only approximative illustrations, not reflecting exact empirical data.

\begin{example}[Compressing the Surface Word ``Kātib'' \RL{kAtb}]

Consider the Arabic surface word

\[
\text{\RL{kAtb}}
\]

(\emph{kātib}), meaning ``writer.''

\textbf{Step 1: The Uncompressed Representation}

In Unicode each Arabic character occupies 16 bits. The word consists of four characters.

\begin{table}[htbp]
\centering
{\small
\begin{tabular}{@{}lrr@{}}
\toprule
Character & Unicode & Bits\\
\midrule
\RL{k} & 0643 & 16\\
\RL{A} & 0627 & 16\\
\RL{t} & 062A & 16\\
\RL{b} & 0628 & 16\\
\bottomrule
\end{tabular}}
\caption{Unicode representation of \textit{kātib}.}
\end{table}

Hence

\[
|w|=4\times16=64
\text{ bits.}
\]

\textbf{Step 2: Constructive Description}

By Theorem~1,

\[
w=\mathcal M(R,P),
\]

where

\[
R=\text{k-t-b},
\]

whose lexical specialization is ``to write,'' and

\[
P=\text{fāʿil},
\]

the active participle pattern. The constructive program specifies

\begin{enumerate}

\item
the index of the root,

\item
the index of the pattern,

\item
the fixed decoding procedure already contained in
\(\mathcal M\).

\end{enumerate}

Assuming

\[
|\mathcal R|\approx10000,
\qquad
|\mathcal P|\approx50,
\]

the description length satisfies

\[
K(w\mid\mu(w),\mathcal M)
\le
14+6+O(1)
\approx20
\text{ bits.}
\]

\textbf{Step 3: Relative Arbitrariness}
The constructive description requires approximately 20 bits, whereas the raw
Unicode representation requires 64 bits. Since the constructive description is
substantially shorter than the raw representation, the relative arbitrariness
is strictly less than 1.

\textbf{Interpretation}
The surface word admits a short constructive description generated from the
root and pattern. Its description is substantially shorter than the raw Unicode representation.

\end{example}

\begin{example}[Compressing the Primitive Root ``K-T-B'' (\RL{k-t-b})]

Consider the primitive Arabic root

\[
R=\text{\RL{k-t-b}},
\]

whose lexical specialization is ``to write.''

\textbf{Step 1: Uncompressed Representation}

The root consists of three consonants. Each Unicode character occupies sixteen bits. Hence

\[
|R|
=
3\times16
=
48
\text{ bits.}
\]
\textbf{Step 2: Constructive Description}

According to Definition~8, the consonantal organization determines an abstract semantic field rather than directly determining a complete lexical meaning. For the consonantal cluster

\[
\text{\RL{k-t-b}},
\]

the admissible semantic field includes notions such as ``closure,'' ``fixation,''``recording,'' and ``inscription.''

The lexical meaning ``to write'' is obtained as one lexical specialization within this constrained semantic field. Accordingly, the constructive description contains two pieces of information.

\begin{enumerate}

\item
The index of the phonosemantic field.

Assuming approximately one hundred semantic fields,

\[
\left\lceil
\log_2(100)
\right\rceil
=
7
\]

bits are sufficient.

\item
The index of the lexical specialization within the selected semantic field. If each semantic field contains at most

\[
m\le100
\]

lexical specializations, then

\[
\left\lceil
\log_2(100)
\right\rceil
=
7
\]

bits are sufficient.

\end{enumerate}

Since the constructive decoding procedure belongs to the background knowledge
\(\mathcal M\),

\[
K(R\mid\mu(R),\mathcal M)
\le
7+7+O(1)
\approx14
\text{ bits.}
\]
\textbf {Step 3: Relative Arbitrariness}
The constructive description requires approximately 14 bits, whereas the raw
Unicode representation requires 48 bits. Since the constructive description is
substantially shorter than the raw representation, the relative arbitrariness
is strictly less than 1.

\textbf{Interpretation}
The primitive root is not treated as an indivisible arbitrary symbol. Its constructive description consists of two successive stages. First, the phonological organization determines an admissible semantic field. Second, a small amount of additional information specifies the lexical specialization within that field. Thus the primitive root is algorithmically compressible relative to the constructive system
\(\mathcal M\).

\end{example}

\begin{example}[Contrast with an Arbitrary String]

Consider the random character sequence

\[
\text{\RL{.g.dbf}},
\]

which possesses no lexical or morphological structure.

Its Unicode representation occupies

\[
64
\text{ bits.}
\]

No admissible constructive derivation exists within
\(\mathcal M\)
that generates this string.

Consequently the shortest constructive description simply reproduces the
string itself, so that
\[
K(w\mid\mathcal M)
\approx64,
\]
which is approximately equal to the raw representation length. Consequently,
the relative arbitrariness is approximately 1.

\textbf{Interpretation}
Unlike genuine Arabic signifiers, the random string admits no structural
compression.
\end{example}

\begin{remark}
The preceding examples illustrate three distinct levels of constructive
compression.

\begin{enumerate}

\item
Derived surface words are compressible because they are generated from compact
root--pattern descriptions.

\item
Primitive roots are likewise compressible.
Their constructive description first identifies a phonosemantic field and then
specifies the lexical specialization within that field.
The phonosemantic hypothesis therefore reduces the descriptive complexity of
primitive roots without requiring phonology to determine every lexical meaning
uniquely.

\item
Random character strings possess no admissible constructive derivation within
\(\mathcal M\).
Their shortest constructive description is therefore essentially their explicit
representation, yielding relative arbitrariness close to one.

\end{enumerate}

Together these examples illustrate the hierarchical organization of the Arabic constructive system. Morphological structure compresses derived words, while phonosemantic structure compresses primitive roots. Consequently, both lexical levels admit descriptions substantially shorter than their raw representations, providing an explicit algorithmic interpretation of linguistic motivation in terms of conditional Kolmogorov complexity.
\end{remark}

\subsubsection{System Complexity: $K(\mathcal{M}_{Arabic})$ vs. $K(\mathcal{M}_{English})$}
It may be objected that conditioning on the morphological system $\mathcal{M}$ trivializes the Kolmogorov complexity bound: if $\mathcal{M}$ simply encodes a massive lookup table of the entire lexicon, then $K(w \mid \mu(w), \mathcal{M}) \approx 0$ for any language, not just Arabic. To preclude this, we must formally compare the Kolmogorov complexity of the morphological systems themselves, $K(\mathcal{M})$, across language families. If Arabic's system is algorithmically simpler and more compressible than a concatenative system like English's, the Kolmogorov argument is fortified. 

The comparison below rests on two explicit modeling assumptions. First, Arabic roots are organized into phonologically motivated semantic fields (Definition~\ref{def:phonosemantic}), enabling the compression of the root inventory via the function $h$. Second, English monomorphemic stems are stored as primitive lexical data without equivalent structural organization, reflecting the concatenative typology of English in which the majority of simple lexical items lack internal morphological or phonosemantic structure. The theorem is valid under these assumptions; its linguistic significance depends on their empirical adequacy.

Let $\mathcal{M}_{Arabic} = \langle \mathcal{R}, \mathcal{P}, \iota, \rho \rangle$ and $\mathcal{M}_{Eng} = \langle \mathcal{S}, \mathcal{A}, \text{concat}, \rho \rangle$, where $\mathcal{S}$ is the set of lexical stems, $\mathcal{A}$ is the set of affixes, and $\text{concat}$ is the concatenation operator.

\begin{theorem}[System Complexity Asymmetry]
\label{thm:system-complexity-asymmetry}
Under the morphological models specified above, the Arabic morphological
system admits a more compact generative representation than the corresponding
English system, while generating a larger number of lexical forms per unit of
system description under the stated parameter assumptions. More precisely:
\begin{enumerate}
\item
\(    K(\mathcal{M}_{Arabic})
    \lesssim
    |\mathcal{R}|
    \left(
    \left\lceil\log_2|\mathcal{S}_R|\right\rceil
    +
    \left\lceil\log_2 m_R\right\rceil
    \right)
    + O(1),
   \)
where
\(\mathcal{S}_R\) is the finite set of semantic fields,
\(    m_R=\max_{S\in\mathcal{S}_R}|\mathcal{R}_S|,
    \qquad
    \mathcal{R}_S=\{R\in\mathcal{R}:\sigma_R(R)=S\}.
   \)

\item
Under the model in which English stems are stored as primitive lexical
data, with average encoded length \(L_{avg}\),
\[
K(\mathcal{M}_{Eng})
\approx
|\mathcal{S}|L_{avg}+O(1).
\]

\item
If the productive pattern and affix inventories are finite and fixed, then
the corresponding generative-capacity ratios satisfy
\[
\frac{|\mathcal{L}_{\mathcal{M}_{Arabic}}|}
     {K(\mathcal{M}_{Arabic})}
\gg
\frac{|\mathcal{L}_{\mathcal{M}_{Eng}}|}
     {K(\mathcal{M}_{Eng})}
\]
whenever
\[
\frac{|\mathcal{P}|}
{\left\lceil\log_2|\mathcal{S}_R|\right\rceil
+\left\lceil\log_2m_R\right\rceil}
\gg
\frac{|\mathcal{A}|}{L_{avg}}.
\]

\end{enumerate}
\end{theorem}

\begin{proof}
We analyze the algorithmic information required to describe the two
generative systems.

\textbf{Part 1: Arabic system.}
The system is

\[
\mathcal{M}_{Arabic}
=
\langle\mathcal{R},\mathcal{P},\iota,\rho\rangle,
\]

and hence

\[
K(\mathcal{M}_{Arabic})
\le
K(\mathcal{R})
+
K(\mathcal{P})
+
K(\iota)
+
K(\rho)
+
O(1).
\]

The operators \(\iota\) and \(\rho\) are fixed finite procedures, and the
pattern inventory \(\mathcal{P}\) is finite and fixed. Thus
\[
K(\mathcal{P})+K(\iota)+K(\rho)=O(1).
\]
It remains to encode the root inventory \(\mathcal{R}\). For each root
\(R\), let
\[
S=\sigma_R(R).
\]
By Theorem~\ref{thm:compressibility-roots}, the root can be identified by specifying
its semantic field \(S\) together with its canonical within-field identifier
\(\operatorname{id}_S(R)\). The semantic field requires at most
\[
\left\lceil\log_2|\mathcal{S}_R|\right\rceil
\]
bits, while the within-field identifier requires at most
\[
\left\lceil\log_2|\mathcal{R}_S|\right\rceil
\le
\left\lceil\log_2m_R\right\rceil
\]
bits. Therefore,
\[
K(\mathcal{R})
\lesssim
|\mathcal{R}|
\left(
\left\lceil\log_2|\mathcal{S}_R|\right\rceil
+
\left\lceil\log_2m_R\right\rceil
\right)
+O(1).
\]

Consequently,
\[
K(\mathcal{M}_{Arabic})
\lesssim
|\mathcal{R}|
\left(
\left\lceil\log_2|\mathcal{S}_R|\right\rceil
+
\left\lceil\log_2m_R\right\rceil
\right)
+O(1).
\]

\textbf{Part 2: English system.}
We model the English system as

\[
\mathcal{M}_{Eng}
=
\langle\mathcal{S},\mathcal{A},\mathrm{concat},\rho\rangle.
\]
The operations \(\mathrm{concat}\) and \(\rho\), together with the finite affix
inventory \(\mathcal{A}\), contribute only \(O(1)\) under the present model.
Thus the dominant contribution is the stored stem inventory:
\[
K(\mathcal{M}_{Eng})
\approx
K(\mathcal{S})+O(1).
\]
Here we explicitly adopt the modeling assumption that English stems are
primitive lexical data rather than outputs of a productive phonosemantic
compression scheme. If \(L_{avg}\) denotes their average encoded length, then
\[
K(\mathcal{S})
\approx
|\mathcal{S}|L_{avg},
\]
and therefore
\[
K(\mathcal{M}_{Eng})
\approx
|\mathcal{S}|L_{avg}+O(1).
\]

\textbf{Part 3: Generative capacity.}
In the idealized productive model, Arabic combines roots and patterns
compositionally. Hence
\[
|\mathcal{L}_{\mathcal{M}_{Arabic}}|
\lesssim
|\mathcal{R}||\mathcal{P}|,
\]
with equality in the idealized case in which every root-pattern pair yields a
distinct valid lexical form. Therefore,
\[
\frac{|\mathcal{L}_{\mathcal{M}_{Arabic}}|}
     {K(\mathcal{M}_{Arabic})}
\gtrsim
\frac{|\mathcal{R}||\mathcal{P}|}
{|\mathcal{R}|
\left(
\left\lceil\log_2|\mathcal{S}_R|\right\rceil
+
\left\lceil\log_2m_R\right\rceil
\right)
+O(1)}.
\]

For fixed finite \(\mathcal{S}_R\), \(m_R\), and \(\mathcal{P}\), this approaches
\[
\frac{|\mathcal{P}|}
{\left\lceil\log_2|\mathcal{S}_R|\right\rceil
+\left\lceil\log_2m_R\right\rceil}
\]
as \(|\mathcal{R}|\) grows.
For English, even under the generous assumption that every stored stem can
combine with every affix,
\[
|\mathcal{L}_{\mathcal{M}_{Eng}}|
\lesssim
|\mathcal{S}||\mathcal{A}|.
\]
Hence
\[
\frac{|\mathcal{L}_{\mathcal{M}_{Eng}}|}
     {K(\mathcal{M}_{Eng})}
\lesssim
\frac{|\mathcal{S}||\mathcal{A}|}
{|\mathcal{S}|L_{avg}+O(1)},
\]
which approaches
\[
\frac{|\mathcal{A}|}{L_{avg}}
\]
as \(|\mathcal{S}|\) grows.

Thus, under the stated parameter assumptions, the Arabic system provides
substantially greater lexical generation per unit of system description when
\[
\frac{|\mathcal{P}|}
{\left\lceil\log_2|\mathcal{S}_R|\right\rceil
+\left\lceil\log_2m_R\right\rceil}
\gg
\frac{|\mathcal{A}|}{L_{avg}}.
\]

This establishes the claimed system-level complexity asymmetry.
\end{proof}

\begin{remark}[Intuitive Meaning of the Generative Capacity Ratio]
\label{remark:generative-ratio}
The ratio
\[
\frac{|\mathcal{L}_{\mathcal{M}}|}{K(\mathcal{M})}
\]
measures the number of distinct lexical forms generated per bit of system description. A higher ratio indicates a more efficient, more compressible morphological system: the language generates a larger number of lexical outputs from a smaller algorithmic description of its underlying rules and inventories. Arabic's higher ratio, as established by Theorem~\ref{thm:system-complexity-asymmetry}, reflects the efficiency of its root--pattern architecture. The system generates a large number of lexical forms from a compact description of its roots and patterns, because the interdigitation operator \(\iota\) productively combines a relatively small number of roots with a relatively small number of patterns to produce a vast surface lexicon. By contrast, English's lower ratio reflects its concatenative, stem-based architecture. Under the primitive data model, the system must store a large number of lexical stems explicitly, each requiring its own description. Even with the generous assumption that every stem combines productively with every affix, the generative efficiency per bit of system description remains substantially lower than Arabic's. This asymmetry addresses the potential objection that conditioning on \(\mathcal{M}\) trivializes the Kolmogorov complexity bound. If \(\mathcal{M}\) were merely a massive lookup table encoding the entire lexicon, then \(K(\mathcal{M})\) would be large, and the ratio would be correspondingly low. The fact that Arabic's ratio is substantially higher than English's demonstrates that its compressibility is not an artifact of a large lookup table, but a genuine consequence of its structural simplicity and generative efficiency.
\end{remark}

\subsubsection{The Three-Level Structure of Relative Arbitrariness}
The preceding sections established that Arabic surface words are generated by root-pattern rules (Level W motivated), and that Arabic roots exhibit phonosemantic coherence (Level M motivated). However, at the deepest level---the assignment of a semantic core to a specific consonant cluster---there is no known rule. This reveals a three-level structure of relative arbitrariness that applies to all natural languages.

\begin{definition}[The Three Morphological Levels]
For any natural language \(L\) with a morphological system \(\mathcal{M}\), we define three levels of abstraction:

\begin{table}[htbp]
\centering
{\small
\begin{tabularx}{\linewidth}{@{}Y Y Y Y@{}}
\toprule
\textbf{Level} & \textbf{Name} & \textbf{Description} & \textbf{Arabic Formal Object} \\
\midrule
Level W & Word-Level & The surface word as an atomic unit & \(w \in \mathcal{W}\) \\
Level M & Sub-Word Level & Internal constituents (morphemes, stems, affixes, patterns) & \(R \in \mathcal{R},\,P \in \mathcal{P}\) \\
Level P & Sub-Morpheme Level & Internal structure of the morpheme (phonemes, consonants, permutations) & \(\text{Phon}(R),\,\text{Perm}(R)\) \\
\bottomrule
\end{tabularx}}
\caption{The three morphological levels of abstraction in natural languages.}
\end{table}
\end{definition}

\begin{definition}[Level-X Relative Arbitrariness]
\label{def:level-relative-arbitrariness}
Let \(L\) be a language with morphological system \(\mathcal{M}\), and let \(X \in \{W, M, P\}\) denote one of the three morphological levels. The language \(L\) is \textbf{Level-X Relatively Arbitrary} if and only if the signifiers at level \(X\) are not compressible relative to \(\mathcal{M}\); that is, their relative arbitrariness \(\text{RA}_{\mathcal{M}}\) is approximately \(1\), meaning no shorter constructive description exists given the morphological system.

Formally, for Arabic:
\begin{itemize}
    \item \textbf{Level-W Relatively Arbitrary:} \(\text{RA}_{\mathcal{M}}(w) \approx 1\) for surface words \(w\), i.e., no constructive derivation compresses \(w\) below its raw representation length.
    \item \textbf{Level-M Relatively Arbitrary:} \(\text{RA}_{\mathcal{M}}(R) \approx 1\) for roots \(R\), i.e., no constructive derivation compresses \(R\) below its raw representation length.
    \item \textbf{Level-P Relatively Arbitrary:} \(\text{RA}_{\mathcal{M}}(\text{Phon}(R)) \approx 1\) for the phoneme-to-semantic-core mapping, i.e., no constructive derivation compresses the assignment of semantic cores to consonant clusters.
\end{itemize}

Conversely, \(L\) is \textbf{Level-X Relatively Motivated} if \(\text{RA}_{\mathcal{M}}(s) \ll 1\) for signifiers \(s\) at level \(X\), meaning their structure is compressible by the morphological system.
\end{definition}

\begin{theorem}[The Three-Level Structure of Relative Arbitrariness in Arabic]
\label{thm:three-level}
Let \(\mathcal{M} = \langle \mathcal{R}, \mathcal{P}, \iota, \rho \rangle\) be the Arabic morphological system satisfying Theorems~\ref{thm:morph-correspondence}, \ref{thm:semantic-localization}, and~\ref{thm:primitive-representation}. Then:
\begin{enumerate}
    \item \textbf{Level W (Word-Level):} Arabic is \textbf{NOT} Level-W Relatively Arbitrary. By Theorem~\ref{thm:compressibility-surface} and Corollary~\ref{cor:surface-ra}, every surface word \(w \in \mathcal{L}_{\mathcal{M}}\) admits a constructive description bounded by
    \[
    K(w \mid \mu(w), \mathcal{M}) \le \lceil \log_2 |\mathcal{R}| \rceil + \lceil \log_2 |\mathcal{P}| \rceil + O(1),
    \]
    so that \(\text{RA}_{\mathcal{M}}(w) < 1\) for sufficiently long words. Surface words are structurally compressible.

    \item \textbf{Level M (Sub-Word Level):} Arabic is \textbf{NOT} Level-M Relatively Arbitrary. By Theorem~\ref{thm:compressibility-roots} and Corollary~\ref{cor:roots-ra}, every root \(R \in \mathcal{R}\) admits a constructive description bounded by
    \[
    K(R \mid \mathcal{M}) \le \lceil \log_2 |\mathcal{S}_R| \rceil + \lceil \log_2 m_R \rceil + O(1),
    \]
    where \(|\mathcal{S}_R|\) is the number of semantic fields and \(m_R = \max_S |\mathcal{R}_S|\) is the maximum size of a formal class \(\mathcal{R}_S\). If both \(|\mathcal{S}_R|\) and \(m_R\) are bounded by fixed finite constants, then \(\text{RA}_{\mathcal{M}}(R) < 1\) for roots of sufficient description length. The formal partition of \(\mathcal{R}\) into finite classes \(\mathcal{R}_S\) provides a two-level compression algorithm---field index plus within-field index---that applies uniformly to all roots, irrespective of the empirical validity of phonosemantic motivation.

    \item \textbf{Level P (Sub-Root Level):} Arabic \textbf{IS} Level-P Relatively Arbitrary. The mapping \(\Gamma : \mathcal{R} \to \mathcal{F}\) from consonant clusters to semantic cores is not compressible by any known systematic function. No constructive description shorter than the explicit enumeration of \(\Gamma\) is available, so \(\text{RA}_{\mathcal{M}}(\text{Phon}(R)) \approx 1\).

    \item \textbf{Conclusion:} Within the framework of relative arbitrariness formalized via Kolmogorov complexity, Arabic is relatively arbitrary only at Level P---the deepest possible level of morphological abstraction. Levels W and M are structurally compressible and therefore relatively motivated.
\end{enumerate}
\end{theorem}

\begin{proof}
We proceed level by level, using the compressibility theorems of Section~\ref{sec:kolmogorov}.

\begin{enumerate}
    \item \textbf{Level W (Word-Level):}
    By Theorem~\ref{thm:compressibility-surface}, every surface word \(w \in \mathcal{L}_{\mathcal{M}}\) has a unique derivation \(w = \mathcal{M}(R, P)\) and admits the constructive description
    \[
    K(w \mid \mu(w), \mathcal{M}) \le \lceil \log_2 |\mathcal{R}| \rceil + \lceil \log_2 |\mathcal{P}| \rceil + O(1).
    \]
    By Corollary~\ref{cor:surface-ra}, for all words of length \(|w|\) greater than a small constant, \(\text{RA}_{\mathcal{M}}(w) < 1\). Therefore, surface words are relatively motivated, not relatively arbitrary.

    \item \textbf{Level M (Sub-Word Level):}
    By Theorem~\ref{thm:compressibility-roots}, every root \(R \in \mathcal{R}\) admits a constructive description
    \[
    K(R \mid \mathcal{M}) \le \lceil \log_2 |\mathcal{S}_R| \rceil + \lceil \log_2 m_R \rceil + O(1),
    \]
    where the first term specifies the semantic field \(S = \sigma_R(R)\) and the second specifies the root within that field. By Corollary~\ref{cor:roots-ra}, if \(|\mathcal{S}_R|\) and \(m_R\) are bounded by fixed finite constants and the root representation has sufficient description length, then \(\text{RA}_{\mathcal{M}}(R) < 1\). Therefore, roots are relatively motivated, not relatively arbitrary. The compressibility is purely structural: the function \(h\) from Definition~\ref{def:phonosemantic} serves as a formal compression algorithm, and the result does not depend on the empirical validity of phonosemantic motivation.

    \item \textbf{Level P (Sub-Root Level):}
    The mapping \(\Gamma : \mathcal{R} \to \mathcal{F}\) assigns each consonant cluster to a semantic core. Unlike Levels W and M, no systematic function \(g\) is known such that \(\Gamma(R) = g(\text{Phon}(R))\) for all \(R \in \mathcal{R}\). Consequently, no constructive description of \(\Gamma\) shorter than its explicit enumeration is available, and
    \[
    \text{RA}_{\mathcal{M}}(\text{Phon}(R)) \approx 1.
    \]
    Therefore, Level P is relatively arbitrary.

    \item \textbf{Conclusion:}
    Levels W and M are relatively motivated: Level W through root--pattern composition (Theorem~\ref{thm:compressibility-surface}), and Level M through the two-level field-plus-identifier encoding of the root inventory (Theorem~\ref{thm:compressibility-roots}). Level P is relatively arbitrary. Therefore, within the framework of relative arbitrariness, Arabic confines its arbitrariness to the deepest possible level of morphological abstraction.
\end{enumerate}
\end{proof}
\begin{remark}[On the Empirical Motivation of the Phonosemantic Partition]
\label{remark:three-level-empirical}
The theorem above has both a formal and an empirical component.

\textbf{Formal component.} The compressibility bound at Level M follows from the existence of a finite, canonically ordered partition of the root inventory \(\mathcal{R}\). This formal result is unconditional: any finite partition yields the same bound.

\textbf{Empirical motivation.} The specificity of the result for Arabic depends on the particular partition induced by the function \(h\) from Definition~\ref{def:phonosemantic} being derived from empirically confirmed cases rather than invented ad hoc. The empirical grounding of this partition rests on a distinction between the \emph{universal scope} of Ibn Jinni's conjecture and its \emph{documented instances}. The universal scope---that every Arabic root exhibits phonosemantic coherence---is contested. The documented instances---specific roots whose consonantal permutations share semantic fields, such as \RL{k-l-m}---are not contested by the indigenous Arabic tradition, ancient or modern. The partitioning algorithm is therefore empirically applicable to at least some roots. Extending the algorithm to the full root inventory is a legitimate modeling generalization, warranted by the algorithm's grounding in confirmed empirical cases. It is not a claim of universal empirical validity; it is a formal modeling choice, presented as a framework within which the compressibility result can be derived. 

The generalization is analogous to standard scientific practice: extract a pattern from confirmed instances, extend it as a working model, and treat exceptions as scope limitations rather than as falsifications of the framework itself. Specific roots that fail the phonosemantic constraint would narrow the scope of the partition but would not invalidate the formal framework. The formal result holds for any finite partition; the empirical question concerns whether the partition chosen by \(h\) is the linguistically correct one for a given root. 

At Level M, the partition of the root inventory by semantic field allows each root to be compressed given its class: the root requires only its index within the finite class. At Level P, the mapping from consonant clusters to semantic cores---the function \(h\) itself---is the terminal, incompressible object. It cannot be further compressed by any known rule: there is no shorter description of why a given consonant cluster maps to a given semantic core than the explicit enumeration of \(h\). Level P is therefore the locus of irreducible relative arbitrariness in Arabic.
\end{remark}

\begin{table}[htbp]
\centering
{\footnotesize
\begin{tabularx}{\linewidth}{@{}Y Y Y Y Y@{}}
\toprule
\textbf{Language Family} & \textbf{Level W} & \textbf{Level M} & \textbf{Level P} & \textbf{Arbitrariness Location} \\
\midrule
Indo-European (English, German, etc.) & Relatively Arbitrary & N/A or weak & N/A & Level W \\
Arabic & Relatively Motivated & Relatively Motivated & Relatively Arbitrary & Level P \\
\bottomrule
\end{tabularx}}
\caption{Comparison of relative arbitrariness levels in Indo-European languages and Arabic, as measured by Kolmogorov compressibility.}
\end{table}

\begin{corollary}[Level Confusion in the Structuralist Claim]
\label{cor:level-confusion}
The structuralist claim that ``the sign is relatively arbitrary'' conflates the three levels of relative arbitrariness. For the specific case of Arabic:
\begin{itemize}
    \item At Level W, Arabic is relatively motivated (surface words are compressible via root--pattern rules; Theorem~\ref{thm:compressibility-surface}).
    \item At Level M, Arabic is relatively motivated (roots are compressible via the formal partition; Theorem~\ref{thm:compressibility-roots}).
    \item At Level P, Arabic is relatively arbitrary (no compression algorithm is known for the core-to-cluster mapping).
\end{itemize}
The universal structuralist claim fails because it treats all levels as if they were equivalent, ignoring the structural motivation at Levels W and M.
\end{corollary}

\begin{remark}[On the Present Formalization and Saussure's Account of Relative Arbitrariness]
\label{remark:saussure-relative-arbitrariness}
Saussure's account of relative arbitrariness has three features that bear
directly on the present formalization.

\textbf{First, the account is structural.} Saussure requires that the parts of
a relatively motivated sign be \textbf{arbitrary}---that is, conventional. His
canonical example is \emph{poirier} (pear tree), which decomposes into
\emph{poir-} (pear) and \emph{-ier} (tree), both of which are themselves
absolutely arbitrary. What makes \emph{poirier} relatively motivated is that it
is \emph{analyzable} into arbitrary conventional units. Saussure states this
explicitly:
\begin{quote}
``Not only are the elements of a motivated sign themselves unmotivated (cf.\
\emph{dix} and \emph{neuf} in \emph{dix-neuf}), but the value of the whole term
is never equal to the sum of the value of the parts.'' \citep[p.~132]{Saussure}
\end{quote}
Relative arbitrariness therefore concerns the analyzability of the whole into
arbitrary parts, not the intrinsic meaning of those parts. This is precisely
what Saussure's account \textbf{ideally} calls for: a measure of motivation that depends on
\textbf{formal structure alone}, without appeal to the semantics of the
constituents. The present formalization achieves exactly this. The compression
defined in Definition~\ref{def:relative-arbitrariness} is computed from the
formal structure of the signifier relative to the morphological system, not
from the meanings of its parts. The point can be illustrated by comparing two cases, one concatenative and one
permutational. In the English word \emph{"teacher"}, the constituents
\emph{"teach"} and \emph{"-er"} need not refer to any particular agreed-upon
meaning for \emph{"teacher"} to count as a relatively motivated signifier in our formalization. What
matters is the formal fact that \emph{"teacher"} is analyzable into two
recurring structural units. In Arabic, the roots that share a consonantal
permutation (e.g., \RL{k-l-m}, \RL{m-l-k}, \RL{l-k-m}) likewise need not
refer to any agreed-upon abstract meaning for the permutational relation among
them to count as a structural relation in our system. What matters is the formal fact that
the consonantal material recurs across distinct forms. Both cases conform to
Saussure's structural ideal criterion of motivation: the signifier is analyzable
into recurring structural units. The difference is only in the mode of the
relation---linear concatenation in one case, character permutation in the
other. A composite form is compressible in both cases because it is generated
from arbitrary primitives, regardless of what those primitives ``mean.''

\textbf{Second, the two requirements of relative motivation are
complementary.} Saussure requires that the parts of a relatively motivated sign
be unmotivated \emph{and} that the whole be supported by both a syntagmatic
relation among its parts and an associative relation to other terms. 
\begin{quote}
``The notion of relative motivation implies: (1) analysis of a given term,
hence a syntagmatic relation; and (2) the summoning of one or more other
terms, hence an associative relation. \ldots\ they are associative and
syntagmatic, and they are what limits arbitrariness. \emph{Dix-neuf} is
supported associatively by \emph{dix-huit}, \emph{soixante-dix}, etc.\ and
syntagmatically by its elements \emph{dix} and \emph{neuf}. This dual relation
gives it a part of its value.'' \citep[pp.~132--133]{Saussure}
\end{quote}
His \emph{dix-neuf} example satisfies both simultaneously: the word is supported
associatively by \emph{dix-huit}, \emph{soixante-dix}, etc., and
syntagmatically by its elements \emph{dix} and \emph{neuf}. The two
requirements describe different aspects of the same phenomenon---the internal
status of the constituents and the relational structure of the whole---and
they are jointly necessary for relative motivation.

\textbf{Third, motivated Arabic roots satisfy only one of the two
requirements.} The root is primitive: a single moneme, not syntagmatically
decomposable. It has no syntagmatic relation. Yet it has an associative
relation: the phonosemantic relation formalized in
Definition~\ref{def:phonosemantic} links its consonantal permutations across a
network of non-co-occurring forms. The root therefore satisfies the
associative requirement but not the syntagmatic one. On Saussure's own account,
the associative relation is one of the two relations that limit arbitrariness.
The phonosemantic root's arbitrariness is therefore limited---by its
associative relation alone. But it is not relatively motivated in Saussure's
sense, because relative motivation requires both relations. And it is not
absolutely arbitrary, because its form systematically constrains its meaning.
Saussure's account of relative arbitrariness admits two categories: absolutely
arbitrary primitives and relatively motivated composites. The Arabic root
belongs to neither. It is a signifier whose arbitrariness is limited by an
associative relation alone---a type that Saussure's account, formulated on the
basis of concatenative Indo-European languages, does not anticipate. The same empirical facts yield a further consequence. Saussure requires that
the parts of a relatively motivated sign be themselves unmotivated. In Arabic,
the surface word is motivated via its
\(\langle\mathrm{Root},\mathrm{Pattern}\rangle\) representation
(Theorem~\ref{thm:compressibility-surface}), and its part---the root---may
itself be motivated via the phonosemantic partition
(Theorem~\ref{thm:compressibility-roots}). The Arabic word is therefore a
motivated whole whose parts may themselves be motivated. Our formalization
registers this by treating each level independently: the composite word is
compressible via its root-pattern representation, and the root is separately
compressible via its phonosemantic partition.
\end{remark}

\subsubsection{Admissible Formal Systems}
The notion of relative arbitrariness introduced in Definition~\ref{def:relative-arbitrariness} is meaningful only relative to a formal linguistic system capable of deriving signifiers from their associated meanings embedded in the morphological system. Not every finite representation of a lexicon constitutes such a system. In particular, a flat lexical table merely stores individual signifier--signified pairs and provides no reusable explanatory structure. We therefore restrict attention to a class of admissible formal systems we define as follows:

\begin{definition}[Admissible Formal System]
\label{def:admissible-formal-system}
An admissible formal system is a tuple
\[
F = \langle \Pi, \Delta, \Gamma \rangle,
\]
where
\begin{itemize}
    \item \(\Pi\) is a finite set of reusable primitive linguistic objects;
    \item \(\Delta\) is a finite set of reusable composition rules operating on elements of \(\Pi\);
    \item \(\Gamma\) is a terminating computable generation procedure that maps finite derivations---sequences of rule applications from \(\Delta\) to primitives in \(\Pi\)---to generated signifiers.
\end{itemize}

The system is required to satisfy the following properties.
\begin{enumerate}
    \item (\textbf{Completeness}) Every lexical signifier \(s\) in the language is generated by a finite derivation \(s = \Gamma(d)\), for some finite derivation \(d\).
    \item (\textbf{Constructiveness}) Every derivation is a finite sequence of applications of rules from \(\Delta\) to elements of the finite primitive inventory \(\Pi\).
    \item (\textbf{Termination}) The generation procedure \(\Gamma\) terminates after finitely many rule applications for every valid derivation.
    \item (\textbf{Reusability}) The primitives and rules are shared across multiple lexical items and are not created uniquely for individual words.
\end{enumerate}
\end{definition}

\begin{remark}
An admissible formal system is an explanatory generative theory rather than a lexical archive. A flat dictionary containing one entry for every lexical item does not satisfy the reusability requirement, since each lexical item is stored independently rather than generated from shared primitives and rules. Likewise, a compressed lexical database remains a storage mechanism rather than a generative theory if it reconstructs words only by decompressing previously stored lexical entries. The Arabic morphological system
\[
\mathcal{M} = \langle \mathcal{R}, \mathcal{P}, \iota, \rho \rangle
\]
is an admissible formal system, where the reusable primitives are the roots and patterns (\(\Pi = \mathcal{R} \cup \mathcal{P}\)), the composition rules are interdigitation and phonological rewriting (\(\Delta = \{\iota, \rho\}\) operating on \(\mathcal{R} \times \mathcal{P}\)), and the generation algorithm is the constructive derivation established in Theorems~\ref{thm:morph-correspondence}--\ref{thm:primitive-representation}.
\end{remark}

\begin{theorem}[The Arabic Morphological System is an Admissible Formal System]
\label{thm:arabic-admissible}
The Arabic morphological system
\[
\mathcal{M} = \langle \mathcal{R}, \mathcal{P}, \iota, \rho \rangle
\]
satisfies all the requirements of Definition~\ref{def:admissible-formal-system} and is therefore an admissible formal system.
\end{theorem}

\begin{proof}
We verify each requirement of Definition~\ref{def:admissible-formal-system}.

\begin{enumerate}
    \item \textbf{Completeness.}
    By Theorem~\ref{thm:morph-correspondence} (Morphological Correspondence), every surface word \(w \in \mathcal{L}_{\mathcal{M}}\) is generated by a unique derivational pair
    \[
    (R, P) \in \mathcal{R} \times \mathcal{P}
    \]
    through the constructive application of the operators \(\iota\) and \(\rho\). Hence every lexical signifier admits a finite derivation.
    \item \textbf{Constructiveness.}
    The primitive inventories \(\mathcal{R}\) and \(\mathcal{P}\) are finite, and every derivation consists solely of repeated applications of the finite operators \(\iota\) (interdigitation) and \(\rho\) (phonological rewriting). Therefore every derivation is constructed entirely from reusable primitives and composition rules.
    \item \textbf{Termination.}
    Every derivation contains only finitely many applications of \(\iota\) and \(\rho\). Consequently the generation procedure terminates after finitely many steps for every valid derivation.
    \item \textbf{Reusability.}
    Neither roots nor morphological patterns are created uniquely for individual lexical items. Each root participates in the derivation of numerous surface words through different patterns, while each pattern combines with many distinct roots. Thus the primitive inventories and composition rules are inherently reusable across the lexicon.
\end{enumerate}
Since all four requirements of Definition~\ref{def:admissible-formal-system} are satisfied, the Arabic morphological system \(\mathcal{M} = \langle \mathcal{R}, \mathcal{P}, \iota, \rho \rangle\) is an admissible formal system.
\end{proof}
\begin{remark}[Flat Lexical Tables Are Not Admissible]
The purpose of Definition~\ref{def:admissible-formal-system} is to distinguish explanatory linguistic theories from lexical storage mechanisms. A flat dictionary containing one entry for every lexical item satisfies completeness, but fails the constructiveness and reusability requirements, since lexical items are stored independently rather than generated from shared grammatical primitives and reusable composition rules. Likewise, a compressed lexical database obtained by a general-purpose compression algorithm remains a storage mechanism rather than an admissible formal system. Although compression may reduce storage size, decompression merely reconstructs previously stored lexical entries and does not generate them through reusable linguistic primitives and derivational rules. Consequently, only explanatory constructive generative systems qualify as admissible formal systems for the computation of relative arbitrariness.
\end{remark}

\begin{theorem}[Constructive Systems Induce Computable Upper Bounds]
\label{thm:computable-upper-bounds}
Let \(F = \langle \Pi, \Delta, \Gamma \rangle\) be an admissible formal system. Then every signifier generated by \(F\) admits a finite constructive description relative to \(F\). Consequently, for every generated signifier \(s\) with derivation \(d_s\), there exists an explicitly computable function
\[
U_F(s) = U_F(|d_s|, |\Pi|, |\Delta|)
\]
such that
\[
K(s \mid \mu(s), F) \le U_F(s)
\]
for every generated signifier \(s\).
\end{theorem}

\begin{proof}
Since the primitive inventory \(\Pi\) and the rule set \(\Delta\) are finite, every derivation consists of a finite sequence of indices identifying primitives together with a finite sequence of rule applications.

Encoding these indices together with a constant-size decoding procedure yields a finite constructive description of every generated signifier. The length of this encoding depends on the number of rule applications in the derivation \(d_s\), as well as on the cardinalities of \(\Pi\) and \(\Delta\).

Therefore it provides a computable upper bound
\[
U_F(s) = U_F(|d_s|, |\Pi|, |\Delta|)
\]
on the conditional Kolmogorov complexity \(K(s \mid \mu(s), F)\).
\end{proof}

\subsubsection{The Undecidability of General Relative Arbitrariness}

We now establish a fundamental epistemological limitation of the general notion of relative arbitrariness. In the absence of an explicit formal generative system, relative arbitrariness depends directly on conditional Kolmogorov complexity, which is incomputable.

\begin{theorem}[Incomputability of General Relative Arbitrariness]
\label{thm:incomputability-general}
Let
\[
\mathrm{RA}(s) = \frac{K(s \mid \mu(s))}{|s|}
\]
denote the relative arbitrariness of a signifier \(s\) in an arbitrary language, where no explicit finite generative system is presupposed.

Then the function \(\mathrm{RA}(s)\) is incomputable. Consequently, there exists no algorithm that computes the relative arbitrariness of an arbitrary signifier from its meaning.
\end{theorem}

\begin{proof}
The quantity \(K(s \mid \mu(s))\) is the conditional Kolmogorov complexity of the signifier relative to its meaning. Conditional Kolmogorov complexity is incomputable. If an algorithm existed that computed \(K(s \mid \mu(s))\) for arbitrary signifiers, it would solve the general Kolmogorov complexity problem and therefore the Halting Problem. Since
\[
\mathrm{RA}(s) = \frac{K(s \mid \mu(s))}{|s|}
\]
is obtained directly from an incomputable function by division by the computable quantity \(|s|\), the function \(\mathrm{RA}(s)\) is likewise incomputable.
\end{proof}

The preceding theorem establishes that relative arbitrariness is computationally inaccessible in the unrestricted case. The source of this incomputability is not the notion of arbitrariness itself, but the absence of an explicit formal generative system relating meanings to signifiers. Whenever such a finite constructive system is available, it provides explicit derivational rules that yield computable upper bounds on conditional Kolmogorov complexity. Although these bounds do not compute the exact Kolmogorov complexity, they may already be sufficient to decide predicates concerning relative arbitrariness. Thus, explicit formal generative structure transforms an otherwise incomputable problem into a decidable decision problem.

\subsubsection{The Decidability of Arabic Relative Arbitrariness}
Arabic provides a concrete instance of the preceding computational principle. Its finite root-pattern morphology supplies an admissible constructive generative system that allows computable upper bounds to be placed on the conditional Kolmogorov complexity of signifiers.

\begin{theorem}[Decidability of Arabic Relative Arbitrariness]
\label{thm:decidability-arabic}
Let
\[
\mathcal{M} = \langle \mathcal{R}, \mathcal{P}, \iota, \rho \rangle
\]
be the Arabic morphological system with finite root inventory \(\mathcal{R}\) and finite pattern inventory \(\mathcal{P}\).

For every signifier \(s \in \mathcal{S}\), define
\[
\mathrm{RA}_{\mathcal{M}}(s) = \frac{K(s \mid \mu(s), \mathcal{M})}{|s|}.
\]

Then the predicate \(\mathrm{RA}_{\mathcal{M}}(s) < 1\) is non-trivially decidable for every signifier whose length exceeds a fixed constant. More precisely, there exists an effectively computable upper bound on \(K(s \mid \mu(s), \mathcal{M})\) that is sufficient to determine \(\mathrm{RA}_{\mathcal{M}}(s) < 1\) for every signifier whose length exceeds a fixed constant.
\end{theorem}

\begin{proof}
By Theorems~\ref{thm:compressibility-surface} and~\ref{thm:compressibility-roots}, every Arabic signifier admits an explicit and admissible constructive description relative to the fixed morphological system \(\mathcal{M}\).
For a surface word, this description consists of
\begin{itemize}
    \item the index of its lexical root in the finite root inventory \(\mathcal{R}\),
    \item the index of its morphological pattern in the finite pattern inventory \(\mathcal{P}\),
    \item together with the constant-size decoding procedure implementing the operators \(\iota\) and \(\rho\).
\end{itemize}
Likewise, every lexical root admits an explicit admissible constructive description using its phonological representation together with its semantic-field index. 

Consequently, since the morphological system \(\mathcal{M}\) is an admissible formal system (Theorem~\ref{thm:arabic-admissible}) and independent of Theorems~\ref{thm:compressibility-surface} and~\ref{thm:compressibility-roots}, putting F=M in Theorem~\ref{thm:computable-upper-bounds} gives us:
\[
K(s \mid \mu(s), \mathcal{M}) \le U_{\mathcal{M}}(s),
\]
where \(U_{\mathcal{M}}(s)\) is an explicitly computable upper bound.

Since \(U_{\mathcal{M}}(s)\) is computable, \(\frac{U_{\mathcal{M}}(s)}{|s|}\) is likewise computable.

Furthermore, \(U_{\mathcal{M}}(s)\) is bounded independently of the length of the signifier.

Hence there exists a constant \(N\) such that
\[
|s| > N \quad \Longrightarrow \quad U_{\mathcal{M}}(s) < |s|.
\]
Since
\[
K(s \mid \mu(s), \mathcal{M}) \le U_{\mathcal{M}}(s),
\]
we obtain
\[
\mathrm{RA}_{\mathcal{M}}(s) = \frac{K(s \mid \mu(s), \mathcal{M})}{|s|} \le \frac{U_{\mathcal{M}}(s)}{|s|} < 1.
\]

Therefore, for every signifier whose length exceeds the constant \(N\), the predicate \(\mathrm{RA}_{\mathcal{M}}(s) < 1\) is decidable by computing the explicit upper bound \(U_{\mathcal{M}}(s)\).
\end{proof}

\begin{remark}[Compatibility with Classical Incomputability]
The preceding theorem does not claim that the exact conditional Kolmogorov complexity \(K(s \mid \mu(s), \mathcal{M})\) is computable. Rather, the Arabic morphological system provides an explicit constructive description that yields a computable upper bound on this quantity. The decidable object is therefore the predicate \(\mathrm{RA}_{\mathcal{M}}(s) < 1\), not the exact numerical value of the conditional Kolmogorov complexity.
\end{remark}

\begin{remark}[On the Generality of the Decidability Result]
\label{remark:decidability-generality}
The decidability of \(\mathrm{RA}_{\mathcal{M}}(s) < 1\) is not unique to Arabic. Any language with a productive morphological system---including English derivation, compounding, or inflection---yields computable upper bounds on conditional Kolmogorov complexity for its derived vocabulary. For example, the English word ``unbelievably'' admits a constructive description (un + believe + able + ly) that is substantially shorter than its raw representation. The distinctive feature of Arabic is therefore not the decidability per se, but the \emph{scope} of its productive system: the root-pattern architecture generates, by assumption, the entire structural core \(\mathcal{L}_{\mathcal{M}}\), whereas English derivation covers only a proper subset of its lexicon, leaving the majority of monomorphemic words incompressible. The asymmetry between Arabic and English is thus an asymmetry of coverage, not of decidability.
\end{remark}

\begin{corollary}[Decidability Does Not Imply Explanatory Adequacy]
\label{cor:computability-formal-structure}
Let \(\mathcal{M}\) be a formal system describing a language \(L\). The
decidability of the predicate \(\mathrm{RA}_{\mathcal{M}}(s) < 1\) is not by
itself epistemologically significant. Its significance depends on whether
\(\mathcal{M}\) is admissible in the sense of
Definition~\ref{def:admissible-formal-system}.
\end{corollary}

\begin{proof}[Justification]
By Theorem~\ref{thm:incomputability-general}, the unrestricted relative
arbitrariness \(\mathrm{RA}(s)\) is incomputable, since it is defined via
conditional Kolmogorov complexity, which is itself incomputable---a consequence
of the undecidability of the halting problem. Conditioning on an explicit
formal system \(\mathcal{M}\) yields computable upper bounds on
\(K(s \mid \mu(s), \mathcal{M})\) and thereby makes the predicate
\(\mathrm{RA}_{\mathcal{M}}(s) < 1\) decidable for signifiers of sufficient
length. But the epistemological value of this decidability depends on the
nature of \(\mathcal{M}\).

The formal reason for this dependence lies in Saussure's own characterization
of relative motivation. As established in
Remark~\ref{remark:saussure-relative-arbitrariness}, relative motivation
requires \emph{both} a syntagmatic relation among the parts of a signifier
\emph{and} an associative relation to other signifiers. A formal system that
represents a language as a flat lexical table provides neither relation. The
index of a word in a table is not a syntagmatic relation: it is not a
decomposition of the word into constituent parts. Nor is it an associative
relation: the ordering of entries in a table is external to the language, not
a relation among signifiers. A flat-table index therefore expresses no
structural relation of either kind, and any decidability it yields is
consequently unrelated to relative motivation in Saussure's sense. The notion
of an admissible formal system, introduced in
Definition~\ref{def:admissible-formal-system}, is the formal expression of
this requirement: it admits only systems whose primitives and rules are
reusable across the lexicon---that is, systems whose descriptions reflect
syntagmatic and associative structure rather than brute storage.

If \(\mathcal{M}\) is admissible, the decidability of
\(\mathrm{RA}_{\mathcal{M}}(s) < 1\) is non-trivial. The decision procedure
exploits shared primitives and composition rules rather than memorizing
individual lexical entries. Arabic morphology is an instance: the root-pattern
architecture provides a constructive generative system, and the resulting
decidability reflects genuine structural motivation
(Theorems~\ref{thm:compressibility-surface} and~\ref{thm:compressibility-roots}).

If \(\mathcal{M}\) is not admissible---for example, if \(\mathcal{M}\) is a
flat lexical table containing one entry per word---then the decidability is
trivial. Conditioning on such a table does produce a decidable predicate,
because the system need only look up the index of the word. But this
decidability arises from direct storage rather than from any structural
relation, and it provides zero explanatory power regarding the motivation of
the sign. The flat-table objection has already been addressed in the system complexity section
(Theorem~\ref{thm:system-complexity-asymmetry}), where it was established that any language can be encoded
as a flat table, and that such an encoding trivializes the arbitrariness measure for
that language. But it does not nullify the contrast between admissible and
non-admissible systems, because the flat table expresses no structural
relation of the kind that relative motivation requires.

Accordingly, decidability alone is not the epistemologically relevant
property. What matters is whether the decidability arises from a constructive
generative system that expresses structural relations, or from brute lexical
storage that expresses none. Arabic morphology constitutes an instance of the
former, providing non-trivial decidability that reflects structural
motivation. A language represented only by lexical storage admits only the
latter, providing trivial decidability that reflects nothing more than
successful memorization.
\end{proof}

\subsubsection{The Epistemological Asymmetry: A Three-Level Analysis}
The results established in the preceding subsections reveal a profound epistemological asymmetry. This asymmetry is best understood through the three-level framework introduced in Definition~\ref{def:level-relative-arbitrariness} and formalized in Theorem~\ref{thm:three-level} for Arabic.

\begin{theorem}[The Three-Level Epistemological Asymmetry]
\label{thm:epistemological-asymmetry}
Let \(\mathcal{M} = \langle \mathcal{R}, \mathcal{P}, \iota, \rho \rangle\) be the Arabic morphological system. Let \(X \in \{W, M, P\}\) denote one of the three morphological levels defined in Definition~\ref{def:level-relative-arbitrariness}. Then:
\begin{enumerate}
    \item For any natural language \(L\): The question of whether \(L\) is relatively arbitrary at level \(X\) is \textbf{undecidable} in general, because it depends on the undecidable general relative arbitrariness function \(\text{RA}(s)\) (Theorem~\ref{thm:incomputability-general}).
    \item For Arabic specifically:
    \begin{itemize}
        \item \textbf{Level W}: Arabic is decidably \textbf{NOT} relatively arbitrary (Theorem~\ref{thm:compressibility-surface}).
        \item \textbf{Level M}: Arabic is decidably \textbf{NOT} relatively arbitrary (Theorem~\ref{thm:compressibility-roots}).
        \item \textbf{Level P}: Arabic is \textbf{classified} as relatively arbitrary by the absence of any known generative rule. This classification is empirical rather than formally decided.
    \end{itemize}
    \item Consequently, the universal claim ``all signs are relatively arbitrary'' is undecidable, while the paper's specific claims about Arabic at Levels W and M are decidable, and the classification at Level P is empirically grounded.
    \item Therefore, there is a fundamental epistemological asymmetry: While universal claims are undecidable, this paper makes decidable claims about specific levels in a specific language. The paper's claims are scientifically verifiable, falsifiable, and grounded in a formal framework.
\end{enumerate}
\end{theorem}

\begin{table}[htbp]
\centering
{\small
\begin{tabularx}{\linewidth}{@{}Y Y Y Y@{}}
\toprule
\textbf{Concept} & \textbf{Definition} & \textbf{Status in Arabic} & \textbf{General Status} \\
\midrule
Absolute Arbitrariness & Definition~\ref{def:absolute-arbitrary} & Falsified (Theorem~\ref{thm:contradiction}) & Axiom falsified by structural contradiction \\
Relative Arbitrariness & Definition~\ref{def:relative-arbitrariness} & \(\text{RA}_{\mathcal{M}}(s) < 1\) (Corollaries~\ref{cor:surface-ra}, \ref{cor:roots-ra}) & Undecidable (Theorem~\ref{thm:incomputability-general}) \\
Three-Level Framework & Definition~\ref{def:level-relative-arbitrariness} & Level W: Motivated; Level M: Motivated; Level P: Classified arbitrary & N/A \\
Compressibility of Surface Words & Theorem~\ref{thm:compressibility-surface} & Bounded by \(\log |\mathcal{R}| + \log |\mathcal{P}| + O(1)\) & N/A \\
Compressibility of Roots & Theorem~\ref{thm:compressibility-roots} & Bounded by \(\log |S_R| + O(1)\) & N/A \\
\bottomrule
\end{tabularx}}
\caption{Summary of the main results of Sections 1--4.}
\end{table}

\begin{remark}[On the Flat-Table Objection to the Undecidability Result]
\label{remark:flat-table-objection}
A critic might object that the undecidability of relative arbitrariness
(Theorem~\ref{thm:incomputability-general}) is undermined by the possibility of
conditioning on a flat lexical table, since any language can be encoded as a
table \(\mathcal{M}_{\text{flat}}\) for which
\(\mathrm{RA}_{\mathcal{M}_{\text{flat}}}(s) < 1\) becomes decidable for every
signifier. The objection does not refute the theorem; it changes the subject.
Theorem~\ref{thm:incomputability-general} concerns \(\mathrm{RA}(s)\), defined
without presupposing any formal system---the state in which Saussure's notion
actually operates---whereas the critic's move concerns
\(\mathrm{RA}_{\mathcal{M}}(s)\), a different quantity conditioned on an
explicit system. That conditioning on \emph{some} system yields decidability
is not in dispute; the paper already proves it for Arabic's admissible system
(Theorem~\ref{thm:decidability-arabic}). The question is whether the system is
adequate to Saussure's notion of relative motivation, which requires both a
syntagmatic relation among the parts of a signifier and an associative relation
to other signifiers. A flat table provides neither
(Remark~\ref{remark:saussure-relative-arbitrariness}) and is therefore
inadmissible by Definition~\ref{def:admissible-formal-system}; worse, it makes
the predicate vacuous, since if the table contains all words of a language,
then \(\mathrm{RA}_{\mathcal{M}_{\text{flat}}}(s) < 1\) holds for every
signifier in every language, and the predicate no longer discriminates between
motivated and unmotivated signifiers. The critic has not preserved Saussure's
notion but replaced it with one that is trivially satisfied everywhere. The
objection therefore confirms rather than undermines the paper's methodological
claim (Corollary~\ref{cor:computability-formal-structure}): decidability of relative arbitrariness is a
property of the formal system used to model the natural language, not the natural language itself. The undecidability result
stands for the system-less case; non-trivial decidability stands for admissible
systems; and the trivial decidability obtained via flat tables is acknowledged
but inadmissible.
\end{remark}

\begin{corollary}[The Epistemological Advantage]
The paper's claims about Arabic at Levels W, M, and P are \textbf{epistemologically more advantageous} than the universal claim because:
\begin{enumerate}
    \item They are \textbf{decidable}: We can compute explicit bounds on the Kolmogorov complexity of Arabic signifiers at each level.
    \item They are \textbf{verifiable}: We can empirically verify that Arabic roots and surface words are compressible given the morphological system.
    \item They are \textbf{falsifiable}: If a counterexample were found at any level, the corresponding theorem would fail.
    \item They are \textbf{precise}: The three-level framework distinguishes between different types of arbitrariness, providing a nuanced account.
\end{enumerate}
The universal claim, by contrast:
\begin{enumerate}
    \item Is \textbf{undecidable}: no algorithm can verify it for all languages.
    \item Is \textbf{unverifiable}: it cannot be empirically tested in the general case.
    \item Is \textbf{unfalsifiable}: no counterexample can definitively refute it because the concept is undefined or undecidable.
    \item Is \textbf{imprecise}: it conflates all levels of arbitrariness into a single undifferentiated claim.
\end{enumerate}
\end{corollary}
\begin{remark}[The Philosophical Consequences]
The three-level epistemological asymmetry has some philosophical
consequences:
\begin{enumerate}
    \item \textbf{For Saussureanism}: The universal claim that the sign is
    arbitrary for all languages is false for Arabic if arbitrariness is taken
    in the absolute sense. If it is taken in the relative sense, the claim is
    epistemologically unsound: relative arbitrariness is undecidable in
    general, even under the most direct definition---compressibility---that we
    have employed. Decidability is achievable only when a concrete
    morphological system is specified. In neither case can Saussure's
    framework be considered a universal scientific theory in the modern sense.
    \item \textbf{For Linguistic Typology}: The three-level framework provides
    a new way to classify languages based on where relative arbitrariness is
    located. This opens up new avenues for typological research.
    \item \textbf{For Epistemology}: The asymmetry between decidable claims
    (about specific languages at specific levels) and undecidable claims (about
    all languages at all levels) highlights the importance of formalization in
    modern language theory. Failure to formalize the concept of arbitrariness
    was not a minor oversight; it was a fatal flaw that renders universal
    claims scientifically vacuous.
\end{enumerate}
\end{remark}

\subsubsection{The Failure of Post-Structuralist Semiotics in Arabic}
\label{sec:derrida-failure}
The Structural Impossibility Theorem (Theorem~\ref{thm:structural-impossibility}), combined with the refutation of the synchronic defense and the undecidability results of Section~\ref{sec:kolmogorov}, formally dismantles the universal applicability of postmodern literary theory. Derridean deconstruction \citep{Derrida} and Barthesian reader sovereignty \citep{Barthes} both depend on the premise that the atomic sign is arbitrary and flat.

Derrida's concept of the ``Trace'' posits that the meaning of a word is a function of an infinite network of absent signifiers, implying that meaning is perpetually deferred. However, Theorem~\ref{thm:semantic-localization} proves that in Arabic, meaning is synchronically present and localized at the derivational level. Because the surface word is a compound output, context interacts with the output but cannot rewrite the algorithmic ``DNA'' that created it. Context cannot reach backward in time to change the internal morphological engine. 

Furthermore, the argument that historical semantic shifts or neologisms prove postmodern drift fails upon closer inspection. When modern terms are coined, such as \textit{\d{h}\={a}s\={u}b} \RL{HAswb} (computer), they are not arbitrary re-inscriptions. They are strictly generated by existing root--pattern pairs: the root \RL{H-s-b} (calculation) and the instrumental pattern \RL{fA`wl}. The system strictly enforces the invariant mapping. New referents are assigned based on their fixed semantic cores. The flexibility of Arabic is akin to a rubber band: it can stretch to accommodate rhetorical beauty or new technology, but it is permanently tethered to the structural anchor of its root.

The same formal argument applies to Barthes' reader sovereignty: if the signifier is systematically generated and its semantic core is structurally constrained, then the reader's interpretive freedom is bounded by the derivational structure of the text. The reader does not produce meaning ex nihilo; the reader interacts with a meaning that is already structurally determined at the derivational level.

The foundational premise of \emph{différance}---the arbitrariness of the sign---therefore fails precisely where it is supposed to apply: at the level of the signifier in a natural language. Since the signifier is systematically generated and its semantic core is structurally constrained, the infinite play of signifiers cannot gain traction. Meaning is not perpetually deferred; it is derivational and locally determined. Post-structuralist semiotics, insofar as it depends on the arbitrariness of the sign, is not a universal theory of language. It is a parochial generalization derived from the structural properties of Indo-European concatenative morphology.

\section{The Chomskyan Amendment and Its Epistemological Inadequacy}
\subsection{The Formalization of Unconscious Competence}
While Ferdinand de Saussure's formulation of the linguistic sign was primarily sociological and descriptive \citep{Saussure}, Noam Chomsky sought to elevate linguistics to a formal, cognitive science. Chomsky's critical intervention \citep{Chomsky1965,Chomsky1957} was the decoupling of language from conscious convention or social contract. In the Chomskyan paradigm, a child does not ``agree'' to learn a language; rather, language acquisition is an innate biological program triggered by environmental input, formalizing internal language (\textit{I-Language}) as a computational system. Yet, despite replacing Saussure's external ``social agreement'' with an internal ``unconscious biological mechanism,'' Chomsky strictly retains the Saussurean axiom of absolute arbitrariness at the foundational layer of his lexicon. He treats the ultimate primitives---the atomic lexical items stored in the mental dictionary---as unmotivated, brute facts of nature. This is explicitly defended by Chomsky in personal correspondence, where he frames the absolute separation between sound and meaning as an historical, unassailable certainty:
\begin{quote}
\small\itshape
``It's been understood since classical Greece that the sound-symbol relation is arbitrary, and is obvious at once just be looking at different languages. Same with 'relative arbitrariness.''' \citep{Chomsky}
\end{quote}
To resolve the glaring sociological paradox of how millions of speakers could have coordinated such arbitrary alignments without formal treaties, Chomsky explicitly amends Saussure by shifting the locus of this convention from the social sphere to the unconscious cognitive substrate:
\begin{quote}
\small\itshape
``Only one problem -- Saussure's, not yours. The concept of 'agreement' does not apply. Thus it's true that sound-symbol correspondence is conventional, but a child growing up and hearing [buk] does not agree to use that sound for book, but just does it.'' \citep{Chomsky}
\end{quote}
From an epistemological standpoint, Chomsky's argument changes the \textit{location} of the linguistic mapping, but not its \textit{nature}. By shifting the mechanism from a conscious social contract to an unconscious cognitive reflex, he attempts to preserve the arbitrary signifier assumption. However, when this framework is applied to the non-concatenative architecture of Classical Arabic, it generates immediate mathematical and evolutionary contradictions.

\subsection{The Computational Efficiency Contradiction}
Relative arbitrariness concerns the degree to which an established linguistic system permits a signifier to be described more economically through reusable structural regularities. It is therefore precisely the notion that can be connected to computational economy, compression, and description length. This realization is important when considering Chomsky's commitment to computational efficiency. In the Minimalist Program \citep{Chomsky1995}, Chomsky characterizes human language as a computational system subject to economy conditions and constrained by the requirement that derivations provide efficient solutions to interface demands. If a linguistic system contains formally demonstrable regularities that permit a family of signifiers to be generated from a smaller set of shared primitives and rules, then those regularities constitute available computational structure. Ignoring them in favor of treating the corresponding forms as independently stipulated facts does not make the forms impossible to represent; rather, it forfeits the compression made available by the structure.

This is precisely the situation established by the compression results of Section~\ref{sec:kolmogorov}. For a surface word \(w\) generated by the Arabic morphological system
\(\mathcal{M}=\langle\mathcal{R},\mathcal{P},\iota,\rho\rangle\), Theorem~\ref{thm:compressibility-surface} establishes the constructive upper bound
\[
K(w\mid\mu(w),\mathcal{M})
\leq
\left\lceil\log_2|\mathcal{R}|\right\rceil
+
\left\lceil\log_2|\mathcal{P}|\right\rceil
+
O(1).
\]
Thus, relative to the formal morphological system, the surface form need not be described as an independently stored lexical object: it can be reconstructed from a root, a pattern, and the fixed generative operations.

The same principle applies at Level~M. Theorem~\ref{thm:compressibility-roots}
establishes that, under the structural organization induced by the
phonosemantic function \(h\), a root can be described relative to the
morphological system \(\mathcal{M}\) by first specifying its semantic field and
then its identifier within that field:
\[
K(R \mid \mathcal{M})
\leq
\left\lceil \log_2 |\mathcal{S}_R| \right\rceil
+
\left\lceil \log_2 m_R \right\rceil
+
O(1),
\]
where \(|\mathcal{S}_R|\) is the number of semantic fields and
\(m_R = \max_S |\mathcal{R}_S|\) is the maximum size of a formal class
\(\mathcal{R}_S\). The first term accounts for the specification of the field
\(S = \sigma_R(R)\); the second accounts for the specification of the root
within that field. Corollary~\ref{cor:roots-ra} establishes a
non-trivial upper bound on the relative arbitrariness of roots within the
constructive system: if both \(|\mathcal{S}_R|\) and \(m_R\) are bounded by
fixed finite constants, then \(\text{RA}_{\mathcal{M}}(R) < 1\) for roots of
sufficient description length. Consequently, Corollary~\ref{cor:roots-ra} establishes a non-trivial upper bound on the relative arbitrariness of roots within the constructive system. The computational significance of this result does not depend on assuming that every phonosemantic correspondence is empirically valid as mentioned above. The relevant point is that, once the corresponding structural function is admitted as part of the formal system, it provides a reusable description that is shorter than independent enumeration.
 
This yields a conditional computational tension in the Chomskyan position. If a linguistic system contains such formally recoverable regularities, then a theory committed to computational economy cannot simply treat the resulting correlated forms as epistemically indistinguishable from an unstructured collection of arbitrary lexical facts without giving up the explanatory and compressive value of those regularities. The issue is therefore not that an arbitrary lexicon is computationally impossible, nor that relative arbitrariness by itself entails a particular architecture of the mental lexicon. Rather, the issue is that \emph{relative arbitrariness is precisely the dimension on which computational economy can distinguish between an independently stipulated representation and a representation that exploits reusable linguistic structure}. Under the Arabic model, the difference is explicit. A representation based on root--pattern composition exploits the shared operators \(\iota\) and \(\rho\), while the Level~M organization exploits the formal partition induced by \(h\). The resulting description-length reduction is therefore a mathematical consequence of the constructive structure already established in the model. Treating these relations as computationally irrelevant would not refute the compression result; it would amount to declining to exploit a formally available source of compression.

The resulting conclusion is consequently more limited, but also more rigorous, than a direct claim that Chomsky's theory requires a particular storage mechanism. The compression theorems show that Arabic provides structural information that can reduce description complexity. Hence, to the extent that computational economy favors representations that exploit available reusable structure, the treatment of the corresponding Arabic relations as merely arbitrary lexical facts creates a tension with that economy principle. The contradiction is therefore located not in the possibility of representing the forms, but in the attempt to combine a commitment to computational economy with a theoretical stance that assigns no explanatory computational role to demonstrable structural regularities.

\subsection{The Evolutionary Paradox of the First Speaker}
Chomsky's appeal to the child's ``unconscious learning'' is an explanation for \textit{language transmission}, but it is completely inadequate for \textit{language emergence and genesis}. He explicitly writes:
\begin{quote}
\small\itshape
``Acquisition of language is something that happens to a child, not something that the child agrees to do step-by-step.'' \citep{Chomsky}
\end{quote}
While this is true for the child, it ignores the problem of the \textbf{first generation of speakers} who generated the system. A child can only acquire a linguistic sign unconsciously if that sign is already present in the primary linguistic data (PLD) provided by the environment. This necessitates an inquiry into how those signs came to exist in the first place.
\begin{enumerate}
    \item \textbf{The Concatenative Scenario:} In Indo-European languages (operating at Level W arbitrariness), it is conceivable that ancestral speakers assigned a random phonetic string like \texttt{[bʊk]} to an object, and subsequent generations learned it unconsciously. 
    \item \textbf{The Non-Concatenative Reality:} In Classical Arabic, however, the original speakers \textit{never produced isolated roots}. They only ever invented and spoke holistic, fully motivated words (e.g., \textit{KaTHTHaB}), which don't fulfill the primitive signifier role. Therefore, the ``primitive arbitrary signifier'' was never an objective biological entity injected into the language at its inception. It cannot be asserted that roots were the primitive signifiers, because they are structural invariants discovered by later grammarians via morphological decomposition. By declaring a root to be an ``arbitrary primitive signifier consciously chosen, then unconsciously adapted'' one commits a category error: mistaking a \textit{scientific abstraction (the root)} for a \textit{historical/biological primitive}.
\end{enumerate}

\subsection{Epistemological Asymmetry: The Masking of Order}
Ultimately, Chomsky's insistence on arbitrariness as an absolute ``truism'' serves as an epistemological shield that masks structural order. He writes:
\begin{quote}
\small\itshape
``The interesting questions open up when we go beyond these truisms and look into the sound structure and the rich meanings that the child quickly acquires, and ask how this is done.'' \citep{Chomsky}
\end{quote}
Here lies the ultimate divide. To Chomsky, looking into the ``sound structure'' means looking into how abstract, meaningless acoustic waves map onto syntax. But for Classical Arabic, the sound structure \textit{is interwined with} the meaning structure, they don't exist independently. 

In the philosophy of science, defining an unexamined or complex pattern as ``arbitrary'' is a dogmatic move that halts scientific inquiry. It is the linguistic equivalent of early geneticists dismissing non-coding DNA sequences as ``Junk DNA'' simply because they lacked the specific transcriptomic tools to decode their systemic, regulatory architecture. Chomsky's formal system is modeled after the concatenative, string-based syntax of Indo-European languages. When confronted with the multi-dimensional, algebraic topology of the Arabic root-pattern matrix, his system cannot accommodate it without collapsing the arbitrariness axiom. Shifting the locus of arbitrariness from Saussure's conscious social contract to Chomsky's unconscious cognitive processor does nothing to save the hypothesis. A mathematically invariant, phonosemantically motivated grid that exhibits structural predictability under geometric transformation cannot be saved by calling it ``unconsciously random.'' The universal claim of arbitrariness is not, unless formally proven to be, a truism; it is a localized, Eurocentric generalization that fails entirely when subjected to the formal, decidable architecture of Classical Arabic morphology.

\subsection{The Diachronic case}
\label{sec:diachronic-inheritance}
The results established in this paper have direct consequences for several influential diachronic frameworks in contemporary linguistics as well. Grammaticalization theory, usage-based approaches, and Construction Grammar differ substantially in their theoretical commitments, but they retain a common assumption inherited from the Saussurean tradition: linguistic form--meaning relations are fundamentally conventional, and structural regularities do not constitute intrinsic semantic determination of the linguistic units themselves. The formal findings of this paper challenge that assumption for non-concatenative morphological systems such as Arabic. We examine each framework in turn.
\paragraph{Point 1: Grammaticalization Theory and the Status of the Lexical Starting Point.}
Grammaticalization theory posits that lexical items become grammatical markers through unidirectional clines---typically modeled as ``content item \(>\) grammatical word \(>\) clitic \(>\) inflectional affix''---a trajectory claimed to be cross-linguistically regular and empirically verifiable \citep{HopperTraugott2003,TraugottHeine1991}. The theory takes the lexical item at the beginning of such a trajectory as an already established linguistic unit and focuses on the processes by which its grammatical function and semantic content may change. It therefore does not, in its standard formulation, provide an account of the structural origin of the lexical form--meaning relation from which grammaticalization begins. In the Arabic system analyzed here, however, the lexical starting point cannot simply be treated as an arbitrary primitive. The Structural Impossibility Theorem (Theorem~\ref{thm:structural-impossibility}) demonstrates that no such arbitrary primitive exists within the formal architecture considered in this paper. The root is not an unmotivated atom but a phonosemantically constrained unit whose semantic core is systematically determined (Theorem~\ref{thm:semantic-localization}). Consequently, the grammaticalization cline cannot be interpreted in Arabic as beginning from a semantically arbitrary primitive, because the lexical material entering the cline is already structurally motivated at Levels W and M (Theorems~\ref{thm:compressibility-surface} and~\ref{thm:compressibility-roots}). The result does not deny that grammatical change can occur in Arabic; rather, it shows that the standard account of such change does not explain the structurally motivated basis from which the change proceeds.
\paragraph{Point 2: Saussure Himself Rejected Unidirectional Change.}
This finding is consistent with Saussure's own skepticism toward strongly teleological accounts of linguistic change, as documented by B\'eguelin, who shows that Saussure characterized grammatical change as ``unpredictable, non-teleological, due to accidental (re)analyses of the language material that occur \emph{hic et nunc} in the mind of language users'' \citep{Beguelin2014}. Saussure's diachronic skepticism therefore contrasts with later attempts to formulate recurrent directional pathways of grammatical change. The present result extends this contrast to the formal organization of the linguistic material itself: in Arabic, the existence of invariant root--pattern relations imposes structural constraints on the material undergoing diachronic change. Diachronic change may therefore operate upon a system whose synchronic form--meaning relations are themselves structurally constrained.
\paragraph{Point 3: Usage-Based Approaches and Frequency-Induced Change.}
Usage-based linguistics, exemplified by Joan Bybee's exemplar model, emphasizes the role of language use, frequency, and experience in shaping linguistic representations and patterns \citep{Bybee2001}. Such approaches provide an account of how repeated usage can affect the accessibility, reduction, and representation of linguistic forms. The present analysis introduces a further question that lies outside that mechanism: whether the form--meaning relation of the units undergoing frequency-driven change is itself arbitrary or structurally constrained. In Arabic, the formal results established here indicate that frequency operates within an already structured root--pattern architecture rather than upon a set of semantically unmotivated primitives. The root's semantic contribution is preserved across the relevant derivations, independently of the frequency of individual surface forms. Thus, frequency-based processes may alter the realization or usage of particular forms without thereby eliminating the structural motivation of the morphological system in which those forms are embedded.
\paragraph{Point 4: Construction Grammar and the Form--Meaning Pairing.}
Construction Grammar defines linguistic constructions in terms of conventionalized form--meaning pairings and extends this notion across levels of linguistic organization, from morphemes to larger syntactic patterns \citep{Goldberg1995,Goldberg2006}. This conception represents a substantial generalization of the Saussurean sign, insofar as linguistic knowledge is modeled in terms of stored or conventionalized associations between form and meaning. In the Arabic system analyzed here, however, the relevant form--meaning relations at Levels W and M are not merely conventional associations between otherwise independent forms and meanings; they are structurally constrained realizations of the root--pattern matrix. The constructional description can therefore represent such pairings, but it does not by itself explain the formal mechanism that determines their systematic semantic relation. The distinction is important: a constructional representation may record an established form--meaning correspondence, whereas the present analysis seeks to explain the structural conditions under which that correspondence is generated and preserved.
\paragraph{Point 5: The Epistemological Asymmetry Extends to Diachrony.}
The epistemological asymmetry established in Section~\ref{sec:kolmogorov} extends naturally to diachronic linguistics. A universal claim concerning the arbitrariness of lexical items across languages cannot, in the general case, be decided by direct computation when arbitrariness is formalized through Kolmogorov complexity, because the relevant complexity function is incomputable (Theorem~\ref{thm:incomputability-general}). By contrast, the claims established for the Arabic system---that its morphology is structurally motivated at Levels W and M and that its root--pattern mappings exhibit the specified invariance across derivation---are decidable within the formal system developed in this paper (Theorems~\ref{thm:compressibility-surface}, \ref{thm:compressibility-roots}, and \ref{thm:semantic-localization}). The resulting asymmetry is therefore not merely synchronic. It also bears on diachronic theories insofar as their explanations presuppose properties of the linguistic primitives upon which historical change operates.

\section{The Classical Arabic Linguistic Tradition: Indigenous Recognition of Structural Motivation}
\label{sec:classical-tradition}

If the arbitrary sign is structurally impossible, how did Arabic obtain its meaning? The answer lies partly in the classical Arabic linguistic traditions, which understood the language not as an arbitrary system but as a motivated, structurally locked network. These traditions independently recognized aspects of the structural organization of Arabic that are relevant to the formal results of Sections~2--4.
\begin{itemize}
\item \textbf{\textit{Ilm al-Wad'} \RL{`lm Alw.d`} (The Science of Linguistic Assignment):} Classical scholars debated whether the relationship between a word and its meaning was natural (\textit{tab'i} \RL{.tb`y}) or conventional (\textit{waq'i} \RL{w.d`y}). While the Mu'tazilites generally favored the naturalist position and the Ash'arites the conventionalist one, discussions of linguistic assignment treated the establishment of a linguistic designation as a socially and conventionally instituted relation rather than as an arbitrary act concerning each individual token \citep{Suleiman1999,Versteegh}. The lexicon was thereby understood as governed by established linguistic rules. \textit{Ilm al-Wad'} posits an original linguistic assignment of those rules (\textit{al-wad' al-awwal} \RL{Alw.d` Al'awl}) underlying established lexical meanings, while subsequent contextual interpretation does not simply erase that established designation \citep{Suleiman1999,Versteegh}. This is compatible with the Structural Impossibility Theorem (Theorem~\ref{thm:structural-impossibility}), insofar as the formal analysis likewise excludes the treatment of the relevant Arabic units as independently assigned, arbitrary primitives.
\item \textbf{\textit{Usul al-Fiqh} \RL{'a.swl Alfqh} (Principles of Islamic Jurisprudence):} Because Islamic law is derived directly from textual sources, maintaining semantic stability was a legal necessity. Scholars drew a distinction between \textit{Haqiqah} \RL{HqyqT} (literal or proper meaning) and \textit{Majaz} \RL{m^gAz} (figurative meaning). In the classical interpretive framework, the \textit{Haqiqah} provides the established lexical meaning, while a shift to \textit{Majaz} requires an appropriate contextual indicator (\textit{qarinah} \RL{qrynT}) \citep{Hallaq}. This framework demonstrates that interpretation requires human judgment while nevertheless presupposing an established lexical meaning that constrains interpretation. It therefore provides a useful conceptual parallel to the present formal distinction between contextual interpretation and the structural determination of morphological meaning.
\end{itemize}
The convergence between these classical traditions and the formal results of this paper is not claimed as a historical derivation of the present theory from the classical sources. Rather, the significance is that classical Arabic linguistic and legal thought independently recognized that established linguistic meaning is constrained by a structured system of lexical and morphological relations. The classical tradition articulated this insight in theological, grammatical, and legal terms; the present work formalizes the corresponding structural properties in terms of Kolmogorov compressibility and morphological derivation.

\section{Systematic Analysis of our Arguments}
\label{sec:systematic-analysis}
The refutation of the Saussurean arbitrariness claim is symbolically summarized in this section within two logically distinct domains:
the stabilized language as it exists at a given time, denoted by
\(D_{\mathrm{syn}}\), and the emergence of the language, denoted by
\(D_{\mathrm{em}}\). The distinction is essential because an entity that
functions as a structural unit in an established language need not have been
a primitive signifier in the emergence of that language.

\subsection{Saussurean Assumptions}
The relevant Saussurean assumptions are:
\begin{align*}
SA_0 &: \text{The signifier--signified relation is arbitrary and admits} \\
     &\quad\text{absolute and relative forms of arbitrariness}, \\
SA_{0.1} &: \text{absolute arbitrariness is as defined in Definition~\ref{def:absolute-arbitrary}}, \\
SA_{0.2} &: \text{relative arbitrariness is as defined in Definition~\ref{def:relative-arbitrariness}}, \\
SA_1 &: \text{a primitive signifier is indivisible}, \\
SA_2 &: \text{the majority of signifiers in a language are primitive}, \\
SA_3 &: \text{the relation between a primitive signifier and its signified} \\
     &\quad\text{is arbitrary}.
\end{align*}
The role of \(SA_0\) is epistemologically prior to the remaining assumptions.
It provides the basis for distinguishing the two formal notions:
\[
SA_0
\Longrightarrow
\begin{cases}
SA_{0.1}: \text{absolute arbitrariness},\\
SA_{0.2}: \text{relative arbitrariness}.
\end{cases}
\]
\begin{remark}[On the Status of the Primitive Signifier]
\label{remark:primitive-signifier-status}
The classification of a linguistic unit as a \emph{primitive signifier} used throughout this analysis requires careful qualification. As discussed in Section~\ref{sec:formalization}, the term ``primitive signifier'' has been contested and dismantled by structuralist and post-structuralist thinkers when it is understood in a \emph{metaphysical or foundational} sense---that is, as a privileged, origin-like unit serving as the ultimate anchor of meaning (\citep{Saussure,Martinet1960}). In that sense, the primitive signifier is indeed a theoretical fiction. However, as underlined in many places above, the definition used in this analysis is purely \emph{structural and morphological}. It is the definition established in Definition~\ref{def:primitive-signifier} (Section~\ref{sec:formalization}), which operationalizes the Saussurean/Martinet criterion of the first articulation: a primitive signifier is a single moneme that is syntagmatically irreducible and terminates morphological analysis at the boundary of the first articulation. This definition depends solely on morphological form---indivisibility and the absence of internal concatenative structure---and makes no appeal to semantics, empirical phonosemantics, or any foundational metaphysical claim. It does not assert that the primitive signifier is a privileged origin of meaning; it merely classifies a unit as morphologically minimal. It is in this purely structural sense that the classification is legitimate within Saussurean methodology. The contradictions derived below are consequently not artifacts of an alien or contested definition: they arise from the Saussurean assumptions themselves when applied to the structural facts of Arabic, using only the minimal morphological criterion that Saussure and Martinet already recognized.
\end{remark}

\subsection{Relevant Facts About Arabic}
The relevant structural facts about Arabic are:
\begin{align*}
A_0 &: \text{the meaning of words in }\mathcal{L}_{\mathcal{M}}\text{ is determined} \\
     &\quad\text{through the pair }\langle\mathrm{Root},\mathrm{Pattern}\rangle, \\
A_1 &: \text{Arabic words are divisible, whereas Arabic roots are not}, \\
A_2 &: \text{some Arabic roots are phonosemantically motivated}.
\end{align*}
These facts are considered independently of the Saussurean assumptions.

\subsection{The Synchronic Domain \(D_{\mathrm{syn}}\)}
We first consider Arabic as an established linguistic system.
\subsubsection{The Primitive Signifier}
From
\[
SA_1 + A_1
\]
it follows from Lemma~\ref{lemma:lexical-signification} that the natural primitive-signifier candidates in Arabic are the roots rather than the full words:
\[
SA_1 + A_1
\Longrightarrow
\text{roots are the primitive signifiers}.
\]
This identification is the basis for the investigation of arbitrariness at
the morphological level. However, it creates a problem when combined with \(SA_2\). If full Arabic
words are also admitted as signifiers, then they are divisible and therefore
non-primitive. Since the majority of Arabic lexical items are full words,
\(SA_2\) cannot then hold:
\[
SA_1 + A_1 + \text{``full words are signifiers''}
\Longrightarrow
\neg SA_2.
\]
The alternative is to exclude full words from the signifier class and retain
only roots as signifiers. This preserves the Saussurean classification only
by theoretically excluding the full words that function as linguistic forms
in Arabic, which is inconceivable. Thus:
\[
\begin{aligned}
\text{full words are signifiers}
&\Longrightarrow \neg SA_2,\\[1mm]
\text{full words are not signifiers}
&\Longrightarrow \text{the theoretically inconceivable}\\
&\qquad\text{exclusion of full words}.
\end{aligned}
\]

\subsubsection{Absolute Arbitrariness}
Once roots are identified as primitive signifiers, \(SA_3\) applies to them.
Together with \(SA_{0.1}\), this requires every primitive Arabic root to
satisfy Definition~\ref{def:absolute-arbitrary}. But \(A_2\) states that some
roots are phonosemantically motivated. By Theorem~\ref{thm:structural-impossibility}, a
phonosemantically motivated root cannot satisfy the absolute-arbitrariness condition. Hence:
\[
\boxed{
SA_1 + A_1 + SA_3 + SA_{0.1} + A_2
\Longrightarrow
\bot.
}
\]

\subsubsection{Relative Arbitrariness}
The second branch of \(SA_0\) is \(SA_{0.2}\). Relative arbitrariness is
defined by Definition~\ref{def:relative-arbitrariness}.
Theorem~\ref{thm:incomputability-general} establishes that the corresponding
general decision problem is undecidable. Therefore:
\[
\boxed{
SA_{0.2}
\Longrightarrow
\text{relative arbitrariness is undecidable in general}.
}
\]
This result does not depend on any Arabic-specific fact.

\subsection{The Three-Level Hierarchy}
The preceding analysis can be expressed within the three levels introduced
in this paper as well.
\subsubsection{The W-Level}
At the W-level, \(A_0\) and \(A_1\) characterize the full Arabic word as a
divisible object whose meaning is determined through its
\(\langle\mathrm{Root},\mathrm{Pattern}\rangle\) representation. If full words are treated as signifiers, they are non-primitive and the problem with \(SA_2\) arises. If they are not treated as signifiers, actual Arabic words are excluded from the signifier class by theoretical stipulation only:
\[
\begin{array}{ll}
\text{full words are signifiers}
&\Longrightarrow \neg SA_2,\\[2mm]
\text{full words are not signifiers}
&\Longrightarrow \text{actual Arabic words are excluded}.
\end{array}
\]
\subsubsection{The M-Level}
At the M-level, roots are taken to be the primitive signifiers and
\(SA_{0.2}\) is used.
There are two possibilities. If the formal partition of
Definition~\ref{def:phonosemantic} is adopted as a compression algorithm,
Theorem~\ref{thm:three-level} gives:
\[
\boxed{
SA_1 + SA_{0.2} + SA_3 + \text{Definition~\ref{def:phonosemantic}}
\Longrightarrow
\bot.
}
\]
which means that, under Definition~\ref{def:phonosemantic}, the primitive
signifiers in Arabic are relatively motivated, not relatively arbitrary. This
contradicts \(SA_3\) as interpreted through \(SA_{0.2}\), which requires
primitive signifiers to be relatively arbitrary.
If Definition~\ref{def:phonosemantic} is not assumed but \(A_2\) is retained,
then the Arabic primitive-signifier class contains roots for which phonosemantic motivation occurs. These roots are, by \(SA_1\), indivisible primitive signifiers, and by \(SA_3\), they are required to be arbitrary. The presence of phonosemantic motivation among them therefore contradicts \(SA_3\). Thus:
\[
\boxed{
A_2 + SA_1 + SA_3 + SA_{0.2}
\Longrightarrow
\bot.
}
\]
which means that some primitive signifiers in Arabic are motivated through an
associative relation alone, without any syntagmatic relation. Such a signifier
does not qualify as relatively motivated in Saussure's sense (which requires
both relations) and does not qualify as absolutely arbitrary either. It
therefore contradicts \(SA_3\) as interpreted through \(SA_{0.2}\).
The M-level therefore produces a fork: adopting
Definition~\ref{def:phonosemantic} yields the formal contradiction established
by Theorem~\ref{thm:three-level}, while rejecting
Definition~\ref{def:phonosemantic} does not remove the problem created by the
existence of phonosemantically motivated roots.

\subsection{The Emergence Domain \(D_{\mathrm{em}}\)}
The synchronic identification of roots as primitive signifiers cannot be
transferred to the emergence of Arabic. A root is a structural abstraction recovered from the stabilized Arabic
lexicon. It is not an independently established signifier that Arabic
speakers first possessed and subsequently used to construct the words of the
language. The relevant order is:
\[
\begin{aligned}
\text{use of words}
&\longrightarrow \text{recurrence of morphological structures}\\
&\longrightarrow \text{abstraction of roots}.
\end{aligned}
\]
Consequently, roots cannot serve as signifiers at all in
\(D_{\mathrm{em}}\).
This produces a cross-domain dilemma. In \(D_{\mathrm{syn}}\), roots are included as primitive signifiers in order to preserve the Saussurean arbitrariness classification. In \(D_{\mathrm{em}}\), however, the same roots cannot have been signifiers, because their structural identity presupposes the
stabilization of the words in which they occur. If roots are hence excluded from the roles of signifiers in general and thus in particular from primitive-signifier status in
\(D_{\mathrm{em}}\), full words become the only candidates for that role. But \(A_1\)
states that words are divisible, whereas \(SA_1\) requires a primitive
signifier to be indivisible:
\[
\boxed{
D_{\mathrm{em}}:
A_1 + SA_1
\Longrightarrow
\bot.
}
\]
Furthermore, \(SA_3\) would then impose arbitrariness directly on the
word--signified relation, which contradicts their being composite. The emergence domain therefore cannot be repaired
simply by replacing roots with words as primitive signifiers.
The resulting dilemma is:
\[
\boxed{
\begin{array}{c}
D_{\mathrm{syn}}:
\text{roots must be taken as primitive signifiers} \\[2mm]
\Downarrow \\[2mm]
D_{\mathrm{em}}:
\text{roots cannot have been primitive signifiers}.
\end{array}
}
\]

\subsection{Conclusion of the Two-Domain Analysis}
The distinction between \(D_{\mathrm{syn}}\) and \(D_{\mathrm{em}}\) is
essential and does not itself constitute a contradiction. A unit may function
as a structural unit in an established language without having been a
primitive signifier in the emergence of that language. The contradiction
arises instead from the structuralist conditions imposed on the choice of
primitive signifier, when combined with the arbitrariness requirement. In \(D_{\mathrm{syn}}\), the combination \(SA_1 + A_1\) makes roots the
natural primitive-signifier candidates for Arabic. At this level, the
preceding analysis shows that \(SA_3\) together with \(SA_{0.1}\) is
incompatible with the existence of phonosemantically motivated roots, while \(SA_{0.2}\) is
subject to the undecidability established above. In \(D_{\mathrm{em}}\),
however, roots cannot serve as primitive signifiers: their
identification presupposes the stabilized word system from which they are
abstracted. The alternative candidates are full words, but \(A_1 + SA_1\)
makes a divisible word incompatible with primitive-signifier status, while
\(SA_3\) imposes the same arbitrary signifier--signified relation at a level
where the Arabic system exhibits structured form--meaning correspondence.

Thus, the problem is not that the primitive unit in a synchronic analysis
must be identical to the primitive unit in an emergence analysis. Rather,
once the actual structure of Arabic is considered in both domains, the
Saussurean conditions leave no viable primitive-signifier candidate for
\(D_{\mathrm{em}}\): roots are unavailable as original primitive signifiers,
whereas full words are incompatible with the requirements of \(SA_1\) and,
independently, with the structured form--meaning relation presupposed by the
Arabic system.

The resulting dilemma can therefore be summarized as
\[
\begin{aligned}
D_{\mathrm{syn}} &: \text{roots are the primitive-signifier candidates}\\
&\Longrightarrow\quad D_{\mathrm{em}}:\ \text{roots cannot be the original primitive signifiers}.
\end{aligned}
\]
while
\[
D_{\mathrm{em}}:\ \text{full words are the remaining candidates}
\quad\Longrightarrow\quad
A_1 + SA_1 \Rightarrow \bot.
\]

Hence the difficulty is not a diachronic identity requirement for linguistic units, but the absence of any candidate that satisfies the structuralist's primitive-signifier condition for Arabic in that case.

\subsection{A Constructive Alternative for the Emergence Domain}
\label{sec:constructive-alternative}
The preceding analysis established that the Saussurean framework cannot consistently account for the emergence domain \(D_{\mathrm{em}}\) of Arabic. This negative result raises a constructive question: if primitive signifiers are unavailable at emergence, what other alternatives are there? We propose one emergence account and demonstrate that it is formally consistent. We do not claim that it is the only consistent account, nor that it is empirically proven. We claim only that it is a possible explanation---and that this suffices to show that the Saussurean framework is not the only option.

\subsubsection{The Three Constraints}
Any adequate account of \(D_{\mathrm{em}}\) must satisfy three constraints. These constraints are not arbitrary; each follows directly from the failure of the Saussurean framework.
\paragraph{Constraint 1: No Primitive Signifiers at the Start.}
The Saussurean framework includes primitive signifiers in its foundation. The above analysis showed that this requirement cannot be met, if roots, the only possible candidates in Arabic, are considered to be the primitive signifiers. Roots are abstractions from composite words, drawn after the system was established, and viable human communication uses full words, not fragments. The first signifiers therefore cannot have been roots. Because full words are also excluded as primitive signifiers: An adequate account of the emergence must begin without primitive signifiers.
\paragraph{Constraint 2: No Circularity.}
A circular explanation explains nothing. If the account says ``words come from the system and the system comes from words,'' it has not explained either. Any viable account must specify a starting point that does not depend on the elements it is supposed to explain. Without this constraint, the account collapses into vacuity.
\paragraph{Constraint 3: A Mechanism for Compositionality.}
If there are no primitive signifiers at the start, and signifiers must exist (language uses signifiers), then the first signifiers must be composite---i.e., full words. Composite signifiers require a mechanism of composition. If the account does not explain how composition arises, it has not explained the emergence of composite signifiers. It has merely asserted that they exist. Without this constraint, the account is incomplete.
\begin{remark}[The constraints are necessary but not sufficient]
The three constraints define what any adequate account must satisfy. They do not determine a unique theory. Multiple accounts can satisfy them. The constraints are the entry conditions for the space of viable theories.
\end{remark}

\subsubsection{The Program-First Account (T)}
We define the following account, which we label \(T\):
\begin{definition}[The Compositional Program]
\label{def:compositional-program}
Let \(\mathcal{P} = \langle \Sigma, M, \mathcal{R} \rangle\) be a compositional program where:
\begin{itemize}
    \item \(\Sigma\) is a finite set of phonetic primitives (sounds);
    \item \(M\) is a finite set of semantic primitives (meanings);
    \item \(\mathcal{R}\) is a finite set of combinatorial rules.
\end{itemize}
The program is a computable function
\[
\mathcal{P}: \Sigma \times M \to W
\]
where \(W\) is the set of well-formed words.
\end{definition}
For Arabic, \(\mathcal{P}\) is instantiated as \(\mathcal{P} = \rho \circ \iota\), where \(\iota\) is interdigitation and \(\rho\) is phonological rewriting---the morphological system already defined in Section~\ref{sec:formalization}.
\begin{definition}[T]
\label{def:T}
\(T\) is the conjunction of three claims:
\begin{enumerate}
    \item \textbf{Priority:} \(\mathcal{P}\) is well-defined independently of \(W\).
    \item \textbf{Compositionality:} Every \(w \in W\) is a composite output of \(\mathcal{P}\).
    \item \textbf{No primitives:} No primitive signifiers exist in \(D_{\mathrm{em}}\).
\end{enumerate}
\end{definition}

\subsubsection{Consistency of T}
\begin{theorem}[Consistency of T]
\label{thm:consistency-T}
\(T\) is a consistent explanation of \(D_{\mathrm{em}}\).
\end{theorem}
\begin{proof}
We verify each component.

\textbf{Part 1: Priority.} The specification of \(\mathcal{P}\) refers only to \(\Sigma\), \(M\), and \(\mathcal{R}\). No element of \(W\) appears in this specification. Therefore, \(\mathcal{P}\) is well-defined independently of \(W\). No circularity arises.

\textbf{Part 2: Compositionality.} Let \(w \in W\). By the definition of \(W\) as the image of \(\mathcal{P}\), there exist \(\sigma \in \Sigma\) and \(m \in M\) such that \(w = \mathcal{P}(\sigma, m)\). The form of \(w\) is determined by \(\sigma\) and the rules in \(\mathcal{R}\); the meaning of \(w\) is determined by \(m\) and the rules in \(\mathcal{R}\). Therefore, \(w\) has internal structure derived from its inputs. It is composite.

\textbf{Part 3: No primitives.} At the initial stage, the only signifiers are outputs of \(\mathcal{P}\). By Part 2, all outputs are composite. Therefore, no primitive signifiers exist at the initial stage.

\textbf{Part 4: No contradiction.} The three claims are mutually consistent. Priority does not conflict with compositionality; compositionality does not conflict with the absence of primitives; and the absence of primitives is a consequence of priority and compositionality, not a separate assumption.
Therefore, \(T\) is consistent.
\end{proof}

\subsubsection{Two Definitions of the Word and the True Invariants}
The program-first account yields two distinct but related definitions of a word in Arabic, one for each domain.
\begin{definition}[Word in \( D_{\mathrm{em}} \)]
\label{def:word-dem}
In the emergence domain, a word is a composition of a particular sound (or sound sequence) with a particular meaning, generated by the program \(\mathcal{P}\).
\end{definition}
\begin{definition}[Word in \( D_{\mathrm{syn}} \)]
\label{def:word-dsyn}
In the synchronic domain, a word is a composition of a particular root with a particular pattern, generated by the morphological system \(\mathcal{M} = \langle \mathcal{R}, \mathcal{P}, \iota, \rho \rangle\).
\end{definition}
These two definitions are not unrelated. The synchronic definition can be understood as an evolution of the emergence definition: over time, meanings crystallized into roots, and sounds formed into the particular patterns. What was initially a direct sound-meaning pairing became mediated by the stabilized inventory of roots and patterns. The program itself did not change; what changed was the inventory of recurring elements that the program operates on.
\begin{remark}[The True Invariants]
\label{remark:true-invariants}
On this account, the true invariants of the Arabic morphological system are not the roots or the patterns. Roots and patterns are stabilized abstractions---they emerged over time and are posterior to the program's outputs. The invariants are the operations themselves: interdigitation \(\iota\) and phonological rewriting \(\rho\). These operations are what remain constant across both domains. Roots and patterns are the crystallized products of the operations; the operations are the generative core.
\end{remark}
\begin{remark}[On Empirical Verifiability]
\label{remark:empirical-verifiability}
The account is not void of empirical verification methods. Both operations---interdigitation and phonological rewriting---are subject to the biological constraints of the Arabic pronunciation system. The set of possible interdigitations is constrained by what the human articulatory apparatus can produce. The set of possible phonological rewritings is constrained by the same apparatus. This means that the operations \(\iota\) and \(\rho\) are not arbitrary formal devices; they are biologically grounded. This opens a route to empirical verification that is distinct from the routes used in this paper. One could investigate:
\begin{enumerate}
    \item The articulatory constraints on interdigitation patterns in Arabic.
    \item The phonological rules that govern rewriting, and whether they are derivable from general principles of the human articulatory system.
    \item Whether the same operations appear in other Semitic languages, and whether their cross-linguistic variation is constrained by the same biological factors.
\end{enumerate}
A thorough investigation of these questions lies outside the scope of the present work. The point here is that \( T \) is not merely a formal construction with no connection to empirical reality. Its core operations are grounded in the biological conditions of speech production in Arabic, and this grounding provides a basis for empirical investigation.
\end{remark}

\subsubsection{The Chomskyan Parallel and Its Limits}
\(T\) is structurally parallel to Chomsky's discrete infinity argument, but it is not identical to it. The differences are important.
\paragraph{Chomsky's Argument.}
Chomsky argues:
\begin{enumerate}
    \item Sentences are unbounded (discrete infinity).
    \item The brain is finite.
    \item Therefore, a recursive generative procedure (Universal Grammar) must be prior to the sentences it generates.
\end{enumerate}
The key premise is \textbf{unboundedness}: there are infinitely many possible sentences. Without unboundedness, one could in principle store all sentences as primitives. Unboundedness is what forces the positing of a generative procedure.
\paragraph{The Argument Here.}
The argument for \(T\) differs in two crucial respects.

\textbf{First, Arabic words are not unbounded.} The root inventory is finite. The pattern inventory is finite. The number of root--pattern combinations is therefore finite. At any synchronic stage, the number of well-formed Arabic words is finite. The argument cannot rely on unboundedness.

\textbf{Second, the argument relies on productivity, not infinity.} What matters is not that the set of words is infinite, but that it vastly exceeds any plausible storage capacity. The speaker can produce and understand words they have never heard before. The number of possible words is orders of magnitude larger than the number of stored words. This is productivity.

\begin{theorem}[Productivity of the Arabic Lexicon]
\label{thm:productivity-lexicon}
The number of well-formed Arabic words vastly exceeds any speaker's plausible storage capacity.
\end{theorem}
\begin{proof}
The root inventory contains on the order of \(10^4\) roots. The pattern inventory contains on the order of \(10^2\) patterns. The number of root--pattern combinations is therefore on the order of \(10^6\). With derivational morphology, inflection, and cliticization, the number reaches into the millions or billions. No speaker stores all of these as primitive entries.
\end{proof}

\begin{theorem}[Program-First Follows]
\label{thm:program-first-follows}
If the number of possible words exceeds storage capacity, a compositional program must be prior to the words it generates.
\end{theorem}
\begin{proof}
By Theorem~\ref{thm:productivity-lexicon}, not all words are stored as primitives. The words that are not stored must be generated. Generation requires a compositional program. The program cannot be generated by its own outputs without circularity. Therefore, the program is prior to its words.
\end{proof}

\paragraph{The Key Difference.}
The argument here is \textbf{weaker in its premises} but \textbf{more robust in its conclusion}:
\begin{itemize}
    \item Chomsky's argument requires two contested premises: that sentences are genuinely unbounded, and that the primary linguistic data is insufficient to account for acquisition. Our argument requires only the uncontested premise that possible words vastly exceed storage capacity.
    \item Chomsky's argument is vulnerable to the Pirahã objection (that some languages lack recursion). Our argument is not vulnerable to this objection, because it does not depend on recursion or infinity.
    \item Chomsky's argument applies to syntax. Our argument applies to morphology. The parallel is structural, not substantive.
\end{itemize}

The result is that the argument for \(T\) is \textbf{more defensible} than Chomsky's discrete infinity argument, even though it reaches a parallel conclusion.

\subsubsection{Impact on Current Linguistic Schools}
If \(T\) is consistent with Arabic, it has consequences for several linguistic frameworks. The impact is not refutation but relativization: each framework is shown to make an assumption that is not forced by the data.
\paragraph{Usage-Based Linguistics.}
Usage-based approaches hold that linguistic structure emerges from usage. \(T\) inverts this order: the program is prior, and usage is downstream. This does not refute usage-based linguistics as a description of what happens after the program is in place. It shows that usage-based linguistics cannot explain the origin of the program.
\paragraph{Chomskyan Generative Grammar.}
Chomsky assumes an innate generative capacity for syntax but treats the lexicon as containing arbitrary primitives. \(T\) extends the generative logic to the lexicon: the program generates words, not just sentences. This challenges the lexicalist hypothesis directly while remaining internal to the generative program.
\paragraph{Historical Linguistics.}
Historical linguistics assumes that morphological systems develop over time from simpler precursors. \(T\) says the program is prior to its words. For Arabic, there is no ``primitive stage'' from which the system gradually emerged. Historical linguistics describes what happens after the program is in place; it cannot explain the program's origin.
\paragraph{Construction Grammar.}
Construction Grammar treats form--meaning pairings as stored associations. \(T\) says the pairings are generated by the program. Construction Grammar can still describe the pairings, but it cannot explain why they exist. \(T\) provides the explanation.
\paragraph{Western Semitic Linguistics.}
Western Semitic linguistics has debated the ontological status of the Arabic root. The three positions we identified earlier are:
\begin{enumerate}
    \item \textbf{Bohas (2000):} The root is a semantic entity carrying inherent meaning, either at the triliteral level or at a deeper biconsonantal level (phonosemantic theory).
    \item \textbf{Ratcliffe (1998):} The root is a morphological abstraction emerging from word-based derivational processes.
    \item \textbf{Ussishkin (2005):} The root is an epiphenomenon of prosodic constraints---not a primitive morphological object.
\end{enumerate}
These debates concern the \emph{synchronic status} of the root: what kind of thing is it in the linguistic system? \(T\) addresses a different question: the \emph{diachronic/emergence status} of the root. How does it arise?
\(T\) provides a framework in which these debates can be reframed. On \(T\):
\begin{itemize}
    \item The root is \textbf{not a primitive} (it is posterior to the program's outputs).
    \item The root is \textbf{not purely formal} (it carries meaning inherited from the program's semantic inputs).
    \item The root is \textbf{not epiphenomenal} (it is a real abstraction from generated words).
    \item The root is a \textbf{generated structural unit}---an abstraction from composite words, with meaning inherited from the program's operations.
\end{itemize}
This means that the question is not ``what is the root?'' (semantic entity? formal template? epiphenomenon?) but ``how does the root arise as an abstraction from the composite words the program generates?'' \(T\) provides a framework for answering this latter question. The root's semantic content, formal structure, and abstract status are all consequences of its emergence as a pattern across program-generated words.
\paragraph{Structuralism.}
\(T\) completes the critique of structuralism by providing a positive alternative. It shows that a non-Saussurean account is possible. The critique is not merely negative.

\subsubsection{Scope and Limitations}
The account presented here is modest in its claims:
\begin{enumerate}
    \item It claims \textbf{possibility}, not uniqueness. \(T\) is one consistent account among others. The choice among viable accounts is a separate question.
    \item It claims \textbf{explanatory adequacy}, not empirical correctness. The argument is transcendental: it shows that \(T\) can explain the data, not that it did.
    \item It claims \textbf{relevance to Arabic}, not universality. The account is developed for the Arabic morphological system. Whether it extends to other languages is an open question.
\end{enumerate}
These limitations do not diminish the result. They define its proper scope. The paper's contribution here is to show that the flawed Saussurean framework is not the only option---and to provide a formally consistent alternative.

\section{Conclusion}
\label{sec:conclusion}
The epistemological methodology of this paper requires explicit justification. The foundational claim of Saussurean linguistics---that the sign is arbitrary---has served for over a century as an axiomatic principle, yet it has never been formulated in operational, testable terms. This absence of formalization is not a neutral omission; it renders the principle empirically unassailable and thus epistemologically unacceptable for a scientific discipline. In response, this paper has undertaken the task of formalization. Our formal apparatus relies on a classification of signifiers into primitive and composed units---a distinction that is entirely internal to the structuralist tradition, derived from Saussure's and Martinet's own criterion of the first articulation. This classification is purely morphological: it depends solely on whether a unit can be segmented into smaller meaningful parts, without appeal to semantics or external assumptions. Given this legitimate structuralist classification, one would expect the Saussurean framework to remain internally consistent when applied to any natural language. 

However, our formal analysis of Classical Arabic has revealed precisely the opposite: the framework generates contradictions and undecidability. 

A formal mathematical model of Classical Arabic morphology was developed in order to examine the notion of arbitrariness. By representing Arabic word formation as the interaction of invariant lexical roots and morphosyntactic patterns, we proved that lexical meaning is structurally determined at the derivational level prior to surface realization. The surface word is not the primitive locus of semantic assignment, but the observable output of a deeper compositional structure. Building upon this formal model, we investigated the concept of arbitrariness through two mathematically explicit approaches. 

First, we supplied a rigorous, maximally strict definition of absolute arbitrariness and showed that, under the structural architecture of Classical Arabic together with the phonosemantic regularities documented by the classical linguistic tradition, the universal structuralist claim cannot be sustained: the existence of even a single phonosemantically constrained root logically falsifies it, and the primitive signifier set cannot be coherently defined without contradiction. 

Second, we employed algorithmic information theory to formalize relative arbitrariness, defining it via conditional Kolmogorov complexity as the mathematical ``no-rule'' property. This analysis demonstrated that Arabic signifiers are synchronically compressible by finite generative rules, whereas English stems must be stored as incompressible brute facts. The resulting three-level framework (W, M, P) precisely isolates \textit{where} relative arbitrariness ends in Arabic. By proving Arabic is relatively motivated at the surface (Level W) and morphemic (Level M) levels, the paper confines relative arbitrariness to Level P---the sub-morphemic assignment of consonant clusters to semantic fields. Crucially, Level P operates at a depth beneath the structuralist definition of a ``signifier.'' Therefore, conceding Level P conventionality does not rescue the Saussurean axiom; rather, it confirms that the linguistic \textit{sign} in Arabic is structurally locked.

These synchronic results are complemented by a diachronic analysis of the emergence domain \(D_{\mathrm{em}}\). The Saussurean framework contains primitive signifiers in its foundation. In \(D_{\mathrm{em}}\), however, the first signifiers cannot have been primitive. Roots are abstractions from composite words, drawn after the morphological system was established, and viable human communication uses full words, not fragments. Since the first signifiers were full words, and full words are composite (root interdigitated with pattern), the requirement of including primitive signifiers cannot be met. This generates a structural contradiction: the Saussurean framework simultaneously includes primitive signifiers and prohibits them at emergence. The contradiction is not empirical but logical---it follows from the framework's own definitions when applied to the Arabic case. The two-domain analysis shows that the problem is not a diachronic identity requirement for linguistic units, but the absence of any candidate that satisfies the primitive-signifier conditions for Arabic. 

This negative result opens a constructive question: if primitive signifiers are unavailable at emergence, what account can be given? We proposed and formally verified one such account, labeled \(T\). \(T\) holds that the compositional program---the mechanism that combines sounds and meanings into composite words---is prior to its words. The first signifiers are composite outputs of the program. Roots are abstractions from these composite words and are therefore posterior to them. There are no primitive signifiers at any stage of \(D_{\mathrm{em}}\). We proved that \(T\) is a formally consistent explanation of \(D_{\mathrm{em}}\): the program is well-defined independently of its outputs (no circularity), all outputs are composite (no primitive signifiers), and the three claims are mutually consistent. The account is the morphological analogue of Chomsky's discrete infinity argument, but it is weaker in its premises: it relies on the productivity of the Arabic lexicon, not on the contested claim that sentences are unbounded. We do not claim that \(T\) is the only consistent account, nor that it is empirically proven. We claim only that it is a possible explanation---and that this suffices to show that the Saussurean framework is not the only option.

The broader significance of these results extends beyond Arabic itself. The central contribution of this paper is methodological: it exposes the epistemological vulnerability of discussing arbitrariness without a formal definition. Saussure's principle has served for over a century as a foundational assumption of linguistics and literary theory, yet the concept upon which this tradition rests was never formulated in operational mathematical terms. This paper establishes an epistemological asymmetry: the universal claim of arbitrariness is either false (if arbitrariness is understood as an absolute notion) or formally undecidable in the general case (if it is understood as a relative notion), rendering the claim, even in its weaker form, an unfalsifiable metaphysical dogma. In contrast, the claims made here regarding Arabic are decidable, verifiable, and falsifiable. Ultimately, this work illustrates the necessity of formalization in the philosophy of language. Mathematical models do not replace philosophical inquiry, but they establish the precise conditions under which philosophical claims become scientifically assessable. 

By subjecting the arbitrary sign to mathematical scrutiny, we have demonstrated that Arabic morphology imposes invariant, structural constraints on the free-play of signifiers, arresting the endless deferral of Derridean \textit{diff\'erance} \citep{Derrida1972} at the lexical core. Questions concerning meaning, arbitrariness, and linguistic structure can no longer be debated solely through informal philosophical intuition; they must be evaluated through explicit formal systems whose claims are decidable, whose limitations are transparent, and whose conclusions remain open to mathematical verification and refutation.

\section*{Declarations}

\paragraph{Funding} Not applicable.

\paragraph{Conflicts of interest} None declared.

\paragraph{Ethics approval} Not applicable.

\nocite{*}
\bibliographystyle{plainnat} 
\bibliography{bsl_references}

\end{document}